\documentclass{article}
\usepackage{microtype}
\usepackage{graphicx}
\usepackage{subfigure}
\usepackage{booktabs} 
\usepackage{multicol}

\usepackage{smile}
\usepackage{comment}
\usepackage{url}

\usepackage{amsmath}
\usepackage{amssymb} 
\usepackage{tabularx}
\usepackage{caption}
\usepackage{xcolor}
\usepackage{mathtools}
\usepackage{mathdots}
\usepackage{algorithm}
\usepackage{algorithmic}
\usepackage[algo2e]{algorithm2e}

\usepackage{siunitx}
\usepackage{etoolbox}
\robustify\bfseries

\usepackage[T1]{fontenc}

\DeclareMathOperator{\E}{\mathbb{E}}

\def \cD {\mathcal{D}}
\def \cX {\mathcal{X}}
\def \cY {\mathcal{Y}}

\newcommand\shortsection[1]{\vspace{6pt}{\noindent\bf #1.}}

\usepackage{lipsum}
\newcommand\blfootnote[1]{%
  \begingroup
  \renewcommand\thefootnote{}\footnote{#1}%
  \addtocounter{footnote}{-1}%
  \endgroup
}

     \PassOptionsToPackage{numbers, compress}{natbib}

    \usepackage[final]{neurips_2021}

\usepackage[utf8]{inputenc} 
\usepackage[T1]{fontenc}    

\usepackage{url}            
\usepackage{booktabs}       
\usepackage{amsfonts}       
\usepackage{nicefrac}       
\usepackage{microtype}      
\usepackage{xcolor}         
\usepackage[colorlinks=true,linkcolor=blue,citecolor=green,urlcolor=red]{hyperref}

\def\mytitle{Understanding the Generalization Benefit of Model Invariance from a Data Perspective}
\title{\mytitle}

\author{%
  Sicheng Zhu*, Bang An*, Furong Huang \\
  Department of Computer Science\\
  University of Maryland, College Park\\
  \texttt{\{sczhu, bangan, furongh\}@umd.edu} \\
}

\begin{document}

\maketitle

\begin{abstract}
Machine learning models that are developed with invariance to certain types of data transformations have demonstrated superior generalization performance in practice. 
However, the underlying mechanism that explains why invariance leads to better generalization is not well-understood, limiting our ability to select appropriate data transformations for a given dataset.
This paper studies the generalization benefit of model invariance by introducing the \emph{sample cover induced by transformations}, i.e., a representative subset of a dataset that can approximately recover the whole dataset using transformations.
Based on this notion, we refine the generalization bound for invariant models and characterize the suitability of a set of data transformations by the \emph{sample covering number induced by transformations}, i.e., the smallest size of its induced sample covers. 
We show that the generalization bound can be tightened for suitable transformations that have a small sample covering number.
Moreover, our proposed sample covering number can be empirically evaluated, providing a practical guide for selecting transformations to develop model invariance for better generalization.
We evaluate the sample covering numbers for commonly used transformations on multiple datasets and demonstrate that the smaller sample covering number for a set of transformations indicates a smaller gap between the test and training error for invariant models, thus validating our propositions.
\end{abstract}
\section{Introduction}
Invariance is ubiquitous in many real-world problems.\blfootnote{* Equal contribution.}
For instance, categorical classification of visual objects is invariant to slight viewpoint changes \cite{Edelman_1995,Anselmi_Leibo_Rosasco_Mutch_Tacchetti_Poggio_2014,Han_Roig_Geiger_Poggio_2020}, text understanding is invariant to synonymous substitution and minor typos \cite{Zhang_Zhao_LeCun_2015,Pruthi_Dhingra_Lipton_2019,Jones_Jia_Raghunathan_Liang_2020}. 
Intuitively, models capturing the underlying invariance exhibit improved generalization in practice \citep{Gens_Domingos_2014,cohen_group_2016,Zaheer_Kottur_Ravanbakhsh_2017,Cohen_Geiger_Weiler_2019,Weiler_Geiger_Welling_2018,Cohen_Geiger_Weiler_2019}.
Such generalization benefit is especially crucial when the data are scarce as in some medical tasks \cite{Winkels_Cohen_2018}, or when the task requires more data than usual as in cases of distribution shift \cite{Sagawa_Koh_Hashimoto_Liang_2020} and adversarial attack \citep{schmidt_adversarially_2018,Yin_Kannan_Bartlett_2019,awasthi_adversarial_2020}.

A commonly accepted intuition attributes the generalization benefit of model invariance to the reduced model complexity, especially the reduced sensitivity to spurious features.
However, a principled understanding of why model invariance
helps generalization remains elusive, thus leaving many open questions. 
Since model invariance may come at a cost (e.g., compromised accuracy, increase computational overhead), given a task, how should we choose among various data transformations under which model invariance guarantees better generalization? 
If existing data transformations are not good enough for a given task, what is the guiding principle to find new ones? 
The lack of a principled understanding limits better exploitation of model invariance to further improve generalization.
In addition, since identifying instructive generalization bound is a central topic in machine learning, we may expect to tighten existing generalization bounds by additionally considering the data-dependent model invariance property.

The many faces of data transformations and model classes pose significant challenges to a principled understanding of model invariance's generalization benefit.
To address this, \citep{Anselmi_Leibo_Rosasco_Mutch_Tacchetti_Poggio_2014,anselmi_invariance_2015,Mroueh_Voinea_Poggio_2015,sokolic_generalization_2017,sannai_improved_2020} characterize the input space and show that certain data transformations equivalently shrink the input space for invariant models, which then simplify the input and improves generalization.
From another perspective, \citep{elesedy_provably_2021,lyle_benefits_2020} directly characterize the function space and show that the volume of the invariant model class is reduced, which then simplifies the learning problem and improves generalization.
These understandings provide valuable insights, yet they may become less informative on high-dimensional input data 
or require model invariance to be obtained exclusively via feature averaging. Some certain assumptions on data transformations (e.g., finiteness, group structure with certain measures) also make these understandings less applicable to more general data transformations.

In this paper, we derive generalization bounds for invariant models based on the sample cover induced by data transformations and empirically show that the introduced notion can guide the data transformation selection.
Different from previous understandings, we first identify a data-dependent property of data transformations in a model-agnostic way, and then establish its connections with the refined generalization bounds of invariant models. 
The analysis applies to more general data transformations regardless of how model invariance is obtained and naturally provides model-agnostic guidance for data transformation selection.
We summarize our contributions as follows.

\shortsection{Contributions}
\begin{figure}
     \centering
    \subfigure[$\cG=\emptyset$]{\includegraphics[width=0.49\textwidth]{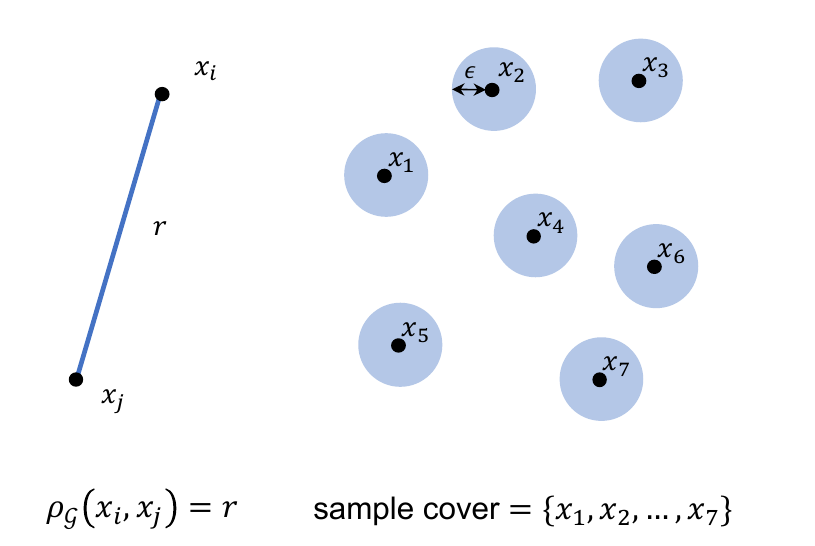}}
    \subfigure[$\cG=\{\text{"rotation"}\}$]{\includegraphics[width=0.49\textwidth]{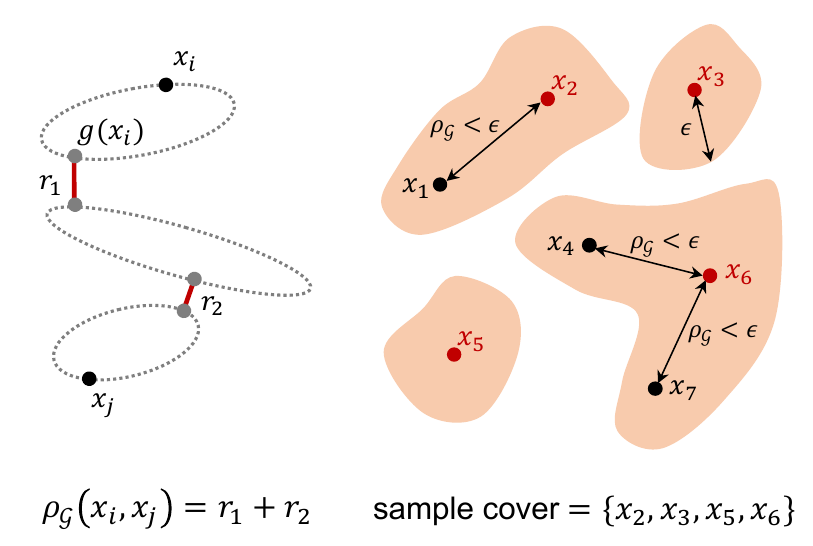}}
    \hfill
    \caption{Illustration of the pseudometric and sample cover induced by data transformations.}
    \label{fig:diagram}
\end{figure}
At the core of our understanding is the notion of sample cover induced by data transformations, defined informally as a representative subset of a dataset that can approximately recover the whole dataset using data transformations (illustrated in Figure~\ref{fig:diagram}).
We show that this notion identifies a data-dependent property of data transformations which is related to the generalization benefit of the corresponding invariant models.
Under a special setting of the sample cover, we first bound the model complexity of any invariant and output-bounded model class in terms of the sample covering numbers. 
Since this general bound requires a restrictive condition on data transformations in order to be informative, we then assume the model Lipschitzness to relax the requirement and refine the model complexity bound for invariant models.
Finally, we outline a framework for model-invariance-sensitive generalization bounds based on the invariant models' complexities, and use it to discuss the generalization benefit of model invariance.

Given the usefulness of sample cover in the analysis, we propose an algorithm to empirically estimate the sample cover. This algorithm exactly verifies whether a given subset of a sample forms a valid sample cover, and always estimates a sample covering number that upper-bounds the ground truth.
Inspired by our analysis, we also propose to use the sample covering number as a suitability measurement for practical data transformation selections.
This measurement is data-driven, widely applicable, and empirically correlates with invariant models' actual generalization performance.
We discuss its limitations and empirical mitigation.

To empirically verify our propositions, we first estimate the sample covering number for some commonly used data transformations on four image datasets, including CIFAR-10 and ShapeNet (a 3D dataset).
Under typical settings, the 3D-view transformation induces a much smaller sample covering number than others on ShapeNet, while cropping induces the smallest sample covering number on others datasets.
For those data transformations, we then train invariant models via data augmentation and invariance loss regularization to evaluate the actual generalization benefit.
Results show a clear correlation between smaller sample covering numbers induced by data transformations and the better generalization benefit enjoyed by invariant models.

\section{Preliminaries}
\vspace{-1em}
\shortsection{Data transformations}
We refer to the data transformation as a function from the input space $\cX\rightarrow\cX$, and data transformation\emph{s} as a set of such functions.
Unless otherwise specified, we do not assume data transformations to have group structures since many non-invertible transformations (e.g., cropping) do not fit into a group structure directly. 
For a set of data transformations  $\cG=\{g:\cX\rightarrow\cX\}$ and a data point (also referred to as an example) $\bx\in\cX$, we overload the notion of orbit in group theory and denote by $\cG(\bx)$ the orbit of $\bx$ defined as follows. 
The \emph{orbit} of $\bx$ generated by data transformations $\cG$ is the collection of outputs after applying any transformation $g\in \cG$ on $\bx$: $\cG(\bx)=\{g(\bx)\in \cX:g\in\cG\}$.

\shortsection{Model invariance}
Let $\cD$ be the underlying data distribution and $\text{supp}(\cD)$ be its support. A model $h:\cX\rightarrow\cY$ is said to be \emph{invariant} under data transformations $\cG$ on $\cD$
if $h(g(\bx))=h(\bx)$ for any $\bx\in\text{supp}(\cD)$ and any $g\in\cG$. We refer to a class of invariant models as the \emph{$\cG$-invariant model class}.

\shortsection{Complexity measurements}
\emph{Covering number} and \emph{Rademacher complexity} \cite{mohri2018foundations} are two commonly used complexity measurements for model classes (including neural networks \cite{bartlett_spectrally-normalized_2017}) that can provide uniform generalization bounds. The covering number can also be directly used to upper bound the Rademacher complexity via Dudley's entropy integral theorem \cite{Dudley_1967,mohri_2012_foundations}.

\emph{Covering number.} 
Let $(\cF, d)$ be a (pseudo)metric space with some (pseudo)metric\footnote{A pseudometric is a metric if and only if it separates distinct points, namely $d(x,y)>0$ for any $x\neq y$.} $d$.
An \emph{$\epsilon$-cover} of a set $\cH\subseteq\cF$ is defined as a subset $\hat{\cH}\subseteq\cH$ such that for any $h\in\cH$, there exists $\hat{h}\in\hat{\cH}$ such that $d(h,\hat{h})\le \epsilon$.
The covering number $N(\epsilon,\cH,  d)$ is defined as the minimum cardinality of an $\epsilon$-cover (among all $\epsilon$-covers) of $\cH$. In this paper, we use the concept of covering number both for measuring model class complexities and for defining the sample covering number on datasets. 

\indent \emph{Empirical Rademacher complexity.} Let $\cH$ be a class of functions $h:\cX\rightarrow\RR$. Given a sample $\cS=\{\bx_i \}_{i=1}^n$, the \emph{empirical Rademacher complexity} of model class $\cH$ is defined as: 
$\rR_{\cS}(\cH) = \EE_{\bsigma} \left[\sup_{h\in \cH}\frac{1}{n}\sum_{i=1}^n \sigma_i h(\bx_i)\right] $
where $\bsigma = [\sigma_1, ..., \sigma_n]^\top$ is the vector of i.i.d. Rademacher random variables, each uniformly chosen from $\{-1, 1\}$.

\shortsection{Generalization error and gap}
Let $\cS=\{\bx_i\}_{i=1}^n$ be a sample drawn i.i.d. from some data distribution $\cD$, and $\cH$ be a model class. Given a loss function $\ell: \RR \rightarrow [0,1]$, for a $h \in \cH$, we define the \emph{empirical error} as $R_{\cS}(h) = \frac{1}{n}\sum_{i=1}^n \ell(h(\bx_i),y_i)$, the \emph{generalization error} as $R(h) =  \EE_{(\bx,y)\sim\cD}[\ell(h(\bx),y)]$, and the \emph{generalization gap} as $R(h)-R_{\cS}(h)$.

\section{Generalization Benefit of Model Invariance}\label{sec:benefit}
In this section, we derive the generalization bounds for invariant models by identifying model invariance properties.
We start by introducing the notion of sample cover induced by data transformations and based on it bound the Rademacher complexity of any invariant models with bounded output
(Section~\ref{subsec:definition-sample-cover}).
Then, we assume model Lipschitzness to provide a more informative model complexity bound for any data transformations (Section~\ref{subsec:lipschitz}).
Finally, we provide a framework for model-invariance-sensitive generalization bounds and discuss the generalization benefit of model invariance
(Section~\ref{subsec:generalization-bounds}).

\subsection{Sample Cover Induced by Data Transformations} \label{subsec:definition-sample-cover}
Existing empirical results suggest that, compared with standard models, invariant models may have certain properties reducing their effective model complexities.
To identify such properties, we alternatively identify the related properties of the corresponding data transformations via the notion of \emph{sample cover induced by data transformations}.
We now formalize the introduced notion.

The definition of sample cover relies on the pseudometric induced by the data transformations $\cG$. 
Note that $\cG$ generates an orbit $\cG(\bx)\subseteq\cX$ for each example $\bx\in\cS$. Let $\|\cdot\|$ be any norm on the input space $\cX$. Given a set of transformations $\cG$, we define the $\cG$-induced pseudometric\footnote{Note that $\rho_\cG$ is not a metric since it allows $\rho_\cG(x,y)=0$ for $x\neq y$.} as
\begin{equation}
    \rho_\cG(\bx_1,\bx_2)=\inf_{\gamma\in\Gamma(\bx_1,\bx_2)} \int_\gamma c(\br) ds, \quad \text{where } c(\br)=\left\{ \begin{array}{cc}
    0, & \text{if $\br\in\cup_{\bx\in\cS} \cG(\bx)$} \\
    1, & \text{otherwise}
\end{array} \right. 
\end{equation}
where $ds=\|d\br\|$, and $\Gamma$ denotes the set of all paths (curves) in $\cX$ from $\bx_1$ to $\bx_2$. The $\rho_\cG$ is essentially calculating the line integral along the shortest (if achievable) path $\gamma$ in the scalar field $c$, where $c$ can also be viewed as the "moving cost" function depending on $\cG$. The norm $\|\cdot\|$ here can be selected as any meaningful norm on the input space (e.g., Euclidean norm as in our experiments) and will later be used in defining the model's Lipschitz constant. It can be checked that $\rho_\cG$ satisfies pseudometric axioms.

\begin{definition}[Sample cover induced by data transformations]
Let $(\cX, \rho_\cG)$ be a pseudometric space and $\cS=\{\bx_i\}_{i=1}^n$ be a sample of size $n$.
An \emph{$\epsilon$-sample cover} $\hat{\cS}_{\cG,\epsilon}$ of the sample $\cS$ induced by data transformations $\cG$ at resolution $\epsilon$
is defined as a subset of the sample $\cS$ 
such that for any ${\bx\in\cS}$, there exists $\hat{\bx}\in\hat{\cS}_{\cG,\epsilon}$ satisfying $\rho_\cG(\bx, \hat{\bx})\leq \epsilon$.
\end{definition}

\begin{definition}[Sample covering number induced by data transformations]
\label{def:sample-covering-number}
The \emph{sample covering number} $N(\epsilon, \cS, \rho_\cG)$ induced by data transformations $\cG$ is defined as the minimum cardinality of an $\epsilon$-sample cover:
\begin{equation}
    N(\epsilon, \cS, \rho_\cG) = \min\{ |\hat{\cS}_{\cG,\epsilon}|: \hat{\cS}_{\cG,\epsilon} \text{ is an } \epsilon\text{-sample cover of } \cS \}.
\end{equation}
\end{definition}

Informally, the $\cG$-induced sample cover specifies a representative subset of examples that can approximately recover all the original examples using the given data transformations $\cG$.
This notion is closely related to the \textit{sample compression} \citep{floyd1995sample} which represents a scheme to prove the learnability of concepts through a compressed set of samples. 
While identifying the generalization-related properties of data transformations, this notion is insensitive to other unrelated properties (e.g., finiteness, group structures) and thus applies to any data transformations.

The intuition behind sample cover is that $\cG$-invariant models may have consistent behaviors on an example and its associated approximation in the $\cG$-induced sample cover.
As such, we can analyze the model complexities of invariant models by considering the models' behavior only on the potentially small-sized sample covers.
Indeed, we directly have the following model complexity result. The proof is in Appendix~\ref{sec:app:proofs}. 

\begin{proposition}\label{prop:upper_bound}
Let $\cS=\{\bx_i\}_{i=1}^n$ be a sample of size $n$.
Let $\cH$ be a model class mapping from $\cX$ to $[-B, B]$ for some $B>0$ and is invariant to data transformations $\cG$. Then the empirical Rademacher complexity of $\cH$ satisfy
\begin{align}
\label{eqn:global-bound}
\rR_\cS(\cH) \le 24B \sqrt{\frac{N(0, \cS, \rho_\cG)}{n}}.
\end{align}
\end{proposition}

Proposition~\ref{prop:upper_bound} generally bounds the model complexity of any output-bounded and $\cG$-invariant model class in terms of the sample covering number $N(0, \cS, \rho_\cG)$ induced by $\cG$.
A small $\cG$-induced sample covering number at resolution $\epsilon=0$ thus tightens the model complexity bound for a general class of $\cG$-invariant models.

Note, however, that Proposition~\ref{prop:upper_bound} is informative only when the data transformations $\cG$ yields $N(0, \cS, \rho_\cG) \ll n$ on the sample $\cS$ ---
a condition requiring $\cG$ to be able to exactly recover $\cS$ from a small-sized subset of $\cS$.
This condition is unfortunately too strict to hold for many commonly used data transformations which only generate orbits with measure zero (with respect to the data measure) at most data points. For example, the rotation transformations on CIFAR-10 do not satisfy this condition, since no two images in CIFAR-10 are rotated versions of each other.
To better understand the generalization benefit brought by any data transformations (e.g., rotation), we further assume specific model properties which equivalently expand the orbits in order to get more general results.
We study Lipschitz models in Section~\ref{subsec:lipschitz}, and relegate a sharper (and relatively independent) analysis for linear models under linear data transformations to Appendix~\ref{subsec:linear}.

\subsection{Refined Complexity Analysis of Lipschitz Models}\label{subsec:lipschitz}
This subsection refines the model complexity analysis for Lipschitz models that are invariant.
Characterizing the Lipschitz constant of models has been the focus of a line of work.
For example, the Lipschitz constant of ReLU networks can be upper-bounded by the product of the spectral norms of the weight matrices, considering the worst-case inputs \citep{bartlett_spectrally-normalized_2017,golowich18a}.
Assuming Lipschitzness, the following theorem refines the covering number analysis for invariant models. The proof is in Appendix~\ref{sec:app:proofs}.

\begin{theorem}
\label{thm:lipschitz}
Let $\cS=\{\bx_i\}_{i=1}^n$ be a sample of size $n$.
Let $\cH$ be a model class such that every $h\in\cH$ is $\kappa$-Lipschitz with respect to $\|\cdot\|$ (used in defining the sample cover) and is invariant to $\cG$.
Then the covering number of $\cH$ satisfies
\begin{align}
\label{eqn:thm-lipschitz}
N\big(\tau, \cH, L_2(\PP_\cS)\big) \leq \inf_{\substack{\epsilon\ge0,\hat{\cS}_{\cG,\epsilon}}} N\big(\tau- \kappa\epsilon \sqrt{1-\frac{|{\hat{\cS}_{\cG,\epsilon}}|}{n}} , \cH, L_2(\PP_{\hat{\cS}_{\cG,\epsilon}})\big),
\end{align}
where $\forall{h,g}\in\cH$, the $L_2(\PP_\cS)$ metric is defined as $\|h-g\|_{L_2(\PP_\cS)}=\left(\sum_{\bx\in\cS} \frac{1}{n} \big(h(\bx)-g(\bx)\big)^2\right)^\frac{1}{2}$, and the $L_2(\PP_{\hat{\cS}_{\cG,\epsilon}})$ metric is defined as\footnote{The term $p(\bx)/n$ can be viewed as the probability mass at $\bx$ where the numerator indicates the number of examples that $\bx$ covers. See Appendix~\ref{subsec:proof-thm-3.4} for the formal definition of $p(\bx)$.} $\|h-g\|_{L_2(\PP_{\hat{\cS}_{\cG,\epsilon}})}=\left(\sum_{\bx\in{\hat{\cS}_{\cG,\epsilon}}} \frac{p(\bx)}{n} \big(h(\bx)-g(\bx)\big)^2\right)^\frac{1}{2}$.
\end{theorem}

Theorem~\ref{thm:lipschitz} upper-bounds the covering number of $\cH$ evaluated at the sample $\cS$ by the new covering number evaluated at any sample cover ${\hat{\cS}_{\cG,\epsilon}}$, under a modified metric and at the cost of an additional error term depending on $\epsilon$ and $\kappa$. 
The equality trivially holds by taking ${\hat{\cS}_{\cG,\epsilon}}=\cS$, while by searching over all sample covers with different resolution $\epsilon$ it is possible to tighten the covering number bound for invariant models.
Additionally, Theorem~\ref{thm:lipschitz} leads to a refined version of Dudley’s entropy integral theorem (see Lemma~\ref{prop:lipschitz-refined-Ramemacher}) that bounds the Rademacher complexity of invariant models.

Theorem~\ref{thm:lipschitz} suggests that we may improve existing covering-number-based model complexity analysis by weakening the dependence on input dimensions.
Note that covering numbers that do not yield $N\big(\tau, \cH, L_2(\PP_{\cS})\big) / n \rightarrow 0$ as $n\rightarrow\infty$ are vacuous. 
Therefore, existing covering number results typically avoid linear dependence on $n$ at the cost of (explicitly or implicitly) increased dependence on the input dimension \cite{Zhang_2002}.
With the refined result in Theorem~\ref{thm:lipschitz}, however, a covering number linear in $n$ can now be replaced by one that is linear in a potentially much smaller sample covering number $N(\epsilon, \cS, \rho_\cG)$ and consequently become informative, thus circumvent the increased dependence on input dimensions.
An interesting direction for future work is to instantiate the result in Equation~\ref{eqn:thm-lipschitz} for specific model classes to get more interpretable results.

\subsection{Framework for Model-invariance-sensitive Generalization Bounds}\label{subsec:generalization-bounds}
This subsection presents the framework for generalization bounds sensitive to model invariance.
While the results are straightforward applications of the derived complexities of invariant models, our goal is to justify the selection of suitable data transformations to maximize the generalization benefit.
We start with the generalization analysis of invariant models and then present the framework.

\shortsection{Generalization benefit for invariant models}
The generalization bounds of invariant models follow immediately by applying the Rademacher model complexities (Proposition~\ref{prop:upper_bound}, Proposition~\ref{prop:lipschitz-refined-Ramemacher}, and Theorem~\ref{thm:linear}) to the standard generalization bound (Theorem~\ref{thm:standard-generalization}).
Compared with standard models, invariant models' tightened model complexity bounds already imply their reduced generalization gaps, whereas for reduced generalization error they further need to have low empirical error.
Since enforcing model invariance may simultaneously increase the empirical error, we can use standard model selection techniques (e.g., structural risk minimization \citep{mohri2018foundations}) to select suitable data transformations and control the trade-off.

\shortsection{Model-invariance-sensitive generalization bound}
We outline the generalization bound that identifies model invariance properties based on the derived invariant models' complexities.
It follows by the post-hoc analysis which specifies a proper set of invariant models using the "invariant loss" --- the loss when composed with any model, makes the composition invariant.
For data transformations with group structures, we can construct such loss by averaging (assuming Haar measure) or adversarially perturbing any given loss over the orbits of input examples \citep{lyle_benefits_2020,elesedy_provably_2021}.
Specifically, the adversarial loss with respect to data transformations $\cG$ is defined as
$\tilde{\ell}_\cG \left(h(x),y\right) = \max_{x'\in {\cG}(x)} \ell\left(h(x'),y\right)$, where $\ell$ is any given loss.
Using the adversarial loss, the following proposition provides the model-invariance-dependent generalization bound by applying the model selection framework \citep{mohri2018foundations}. Appendix~\ref{sec:app:dt-construction} further describes a binary coding construction of combinations of data transformation classes.

\begin{proposition}
\label{prop:uniform-bound}
Let $\cS=\{\bx_i\}_{i=1}^n$ be a sample of size $n$.
Let $\cH$ be any given model class and $\ell$ be any given loss.
Suppose we have $K$ sets of group-structured data transformations $\{\cG_1, \cG_2,... , \cG_K\}$. 
Then with probability at least $1-\delta$, the following generalization bound holds for any $h\in\cH$ and any $k\in[K]$:
\begin{align}
  R(h) &\le \frac{1}{n} \sum_{i=1}^n \tilde{\ell}_{\cG_k}(h(\bx_i), y_i) + 4\rR_{\cS}(\tilde{\ell}_{\cG_k} \circ \cH) + \sqrt{\frac{\log k}{n}} +  3\sqrt{\frac{\log\frac{4}{\delta}}{2n}},
\end{align}
\end{proposition}
where $\rR_{\cS}(\tilde{\ell}_{\cG_k} \circ \cH)$ is upper-bounded by the complexity of $\cG_k$-invariant models.
For any model trained on $\cS$, Proposition~\ref{prop:uniform-bound} shows that we can optimize over all selections of data transformations to improve its generalization bound.
Note that the selection of $\cG_k$ is subject to a potential trade-off between the reduced model complexity $\rR_{\cS}(\tilde{\ell}_{\cG_k} \circ \cH)$ and the increased empirical error $\sum_{i=1}^n \tilde{\ell}_{\cG_k}(h(x_i), y_i)$.
Thus, if a suitable $\cG_k$ reduces the model complexity while keeping the empirical error low, then the trained model will benefit from a tightened generalization bound.
This generalization bound does not require the models to be (strictly) invariant and potentially explains the improved generalization of models with trained invariance (e.g., via data augmentation \cite{taylor2018improving, shorten2019survey} or consistency regularization \cite{miyato2018virtual, xie2019unsupervised}).
The difficulty in instantiating Proposition~\ref{prop:uniform-bound} is that the model complexity with adversarial loss may be hard to compute for general data transformations.
Therefore, we discuss more practical data transformation selections based on the sample covering numbers in Section~\ref{sec:select-trans}.

\section{Sample Cover Estimation Algorithm}
\label{sec:estimation}
The sample cover induced by data transformations plays a central role in our understanding of model invariance.
Despite its usefulness, exactly computing the sample cover turns out to be non-trivial in general.
Indeed, computing the transformation-induced metrics can be difficult for continuous data transformations, and finding the \textit{smallest} sample cover is NP-hard.
To address this problem, we propose an algorithm to estimate the sample covering number and find the associated sample cover.
We outline the algorithm and discuss the algorithmic challenges in this section. The algorithmic details appear in Appendix \ref{app:estimate_sample_cover_num}.

\shortsection{Setup}
The estimation algorithm takes as input a sample $\cS$, a set of data transformations $\cG$, and the resolution parameter $\epsilon$. It then returns the estimated sample covering number $N(\epsilon, \cS, \rho_\cG)$ and the associated sample cover $\hat{\cS}_{\cG,\epsilon}$. The estimation algorithm has the following steps.

\shortsection{Step 1}
Compute (or approximate) the direct orbit distance between any two examples in $\cS$.
The direct orbit distance between any two examples $\bx_i, \bx_j \in \cS$ is
\begin{align*}
    d_{\cG}(\bx_i, \bx_j) = \|{\cG}(\bx_i) - {\cG}(\bx_j)\| = \min_{g_1, g_2 \in \cG} \|g_1(\bx_i)-g_2(\bx_j)\|,
\end{align*}
which can be exactly computed for finite transformations (e.g., flipping) with complexity $O(|\cG|^2)$), or can be approximated for continuous transformations (e.g., rotation) via optimization or sampling.

\shortsection{Step 2}
Compute the $\rho_\cG$ distance between any two examples in $\cS$.
Given results in step 1, computing the $\rho_\cG$ distance between any two examples can be formulated as a shortest path problem on a complete graph, where each node represents an example and the cost of each edge is the direct orbit distance computed in step 1 (see formulations in Appendix \ref{app:estimate_sample_cover_num}). 
Note that the shortest path is always included in our finite candidates even though the $\rho_\cG$ distance considers infinitely many paths. This is because any other path outside our finite candidates will be longer than its counterparts (depending on what orbits it intersects) in our finite candidates.
Standard shortest path algorithms solve for all pairs of examples in polynomial time (e.g., via Dijkstra’s algorithm \cite{dijkstra1959note} in $O(n^3)$).

\shortsection{Step 3}
Construct the pairwise distance matrix $[\rho_\cG(\bx_i, \bx_j)]_{i,j}$ and approximate the sample covering number.
This step can be formulated as a set cover problem where each example $\bx$ covers a subset of $\cS$ in which each element's $\rho_\cG$ distance to $\bx$ is less than or equal to $\epsilon$. Our goal is to find a minimum number of those subsets such that their union contains $\cS$. This problem is known to be NP-hard in general
but admits polynomial time approximations \cite{har2011geometric}. In experiments, we use modified k-medoids \cite{park2009simple} clustering method to find the approximation of $N(\epsilon, \cS, \rho_\cG)$ (see Algorithm \ref{alg:cluster}).

Note that the estimated sample covering number returned by the algorithm is always an upper bound of the ground truth, regardless of the approximation error in steps 1 and 3.
When step 1 is exact, the algorithm also exactly verifies whether a given subset of $\cS$ forms a valid sample cover.
In our experiment, step 2 becomes the computation bottleneck for large-sized samples.
We leave improving the scalability as well as evaluating the approximation quality for future work.

\section{Data-driven Selection of Data Transformations} \label{sec:select-trans}
The pool of candidate data transformations on a given dataset may be infinitely large. To maximize the generalization benefit of model invariance, we usually make selections based on expensive cross-validations due to the absence of model-training-free guidance.
Section \ref{sec:benefit} suggests that invariant models may benefit from improved generalization guarantees if the corresponding data transformations induce small sample covering numbers.
Therefore, we propose to use the sample covering number as an empirical suitability measurement to guide the data transformation selection. We discuss its advantages, limitations, and empirical mitigation in this section.

\shortsection{Suitability measurement}
To maximize the generalization benefit of model invariance on a dataset $\cS$, we measure the suitability of data transformations $\cG$ by the sample covering number induced by $\cG$ and favor the small ones.

\shortsection{Advantages}
One advantage of this suitability measurement is that it is model-training-free. 
It provides a-priori guidance depending only on the dataset and the data transformations, thus avoiding expensive cross-validations and fueling the exploration of new types of data transformations.
Another advantage is that it applies to any type of data transformation (including continuous and non-invertible ones) and provides a uniform benchmark.

\shortsection{Limitations and empirical mitigation}
Being model-agnostic also poses two limitations to the suitability measurement.
One limitation is that this suitability measurement, while capturing invariant models' reduced generalization gap, ignores their potentially increased empirical error.
Note that certain data transformations on a dataset may drastically increase invariant models' empirical error and overturn the benefit of a reduced generalization gap.
To mitigate this limitation, we consider two necessary conditions for maintaining low empirical error.
First, the data transformations should preserve the underlying ground-truth labeling. We may use domain knowledge to meet this condition.
Second, the model class should be rich enough to contain a low-error invariant hypothesis. In our experiment, neural networks which are invariant and achieve low training error suffices this condition.

Another limitation is that this suitability measurement ignores the models' potential Lipschitz constant change after enforcing the invariance.
Theorem~\ref{thm:lipschitz} suggests that the generalization benefit enjoyed by invariant models depends on models' Lipschitz constant and can be overturned if enforcing invariance leads to a significantly larger Lipschitz constant.
To mitigate this limitation, we use the fact that we are doing classification tasks and use the label information to heuristically offset the Lipschitz constant increase.
We use the minimum inter-class distance change after applying data transformations to capture the Lipschitz constant change and use it to normalize the sample covering number for better data transformation selections (see Appendix \ref{app:norm_scn}). 

\section{Experiments}\label{sec:experiment}
In this section, we implement the sample cover estimation algorithm and verify the effectiveness of using sample covering numbers to guide the data transformation selection.
We first estimate the sample covering number induced by different types of data transformations on some image datasets. 
Then, we investigate the actual generalization benefit for models invariant to those data transformations and analyze the correlation\footnote{Code is available at https://github.com/bangann/understanding-invariance.}. 

\shortsection{Datasets} 
We report experimental results on CIFAR-10 \cite{krizhevsky2009learning} and ShapeNet \cite{chang2015shapenet} in this section, and relegate results on CIFAR-100 and Restricted ImageNet to Appendix~\ref{app:exp_CIFAR-100}.
ShapeNet is a large-scale 3D data repository that enables us to do more complex data transformations (e.g., change of 3D-view) beyond the common 2D geometric transformations.
The work \cite{r2n2} provides 24 multi-view pre-rendered images for each 3D object in 10 chosen categories.
For convenience, we use those images to approximate the random perturbation of the 3D view. 

\shortsection{Data transformations}
We evaluate some commonly used data transformations with typical parameter settings which we assume to be label-preserving.
We choose \emph{flipping}, \emph{cropping}, and \emph{rotation} on CIFAR-10, and additionally consider the \emph{3D-view} change on ShapeNet.
We use the same data transformations with the same parameter settings during estimating the sample covering number and evaluating the generalization benefit. Appendix~\ref{app:exp} provides more details of our experimental settings.

\subsection{Estimation of Sample Covering Numbers} \label{exp:scn}

\begin{figure}
     \centering
    \subfigure[CIFAR-10]{\includegraphics[width=0.4\textwidth]{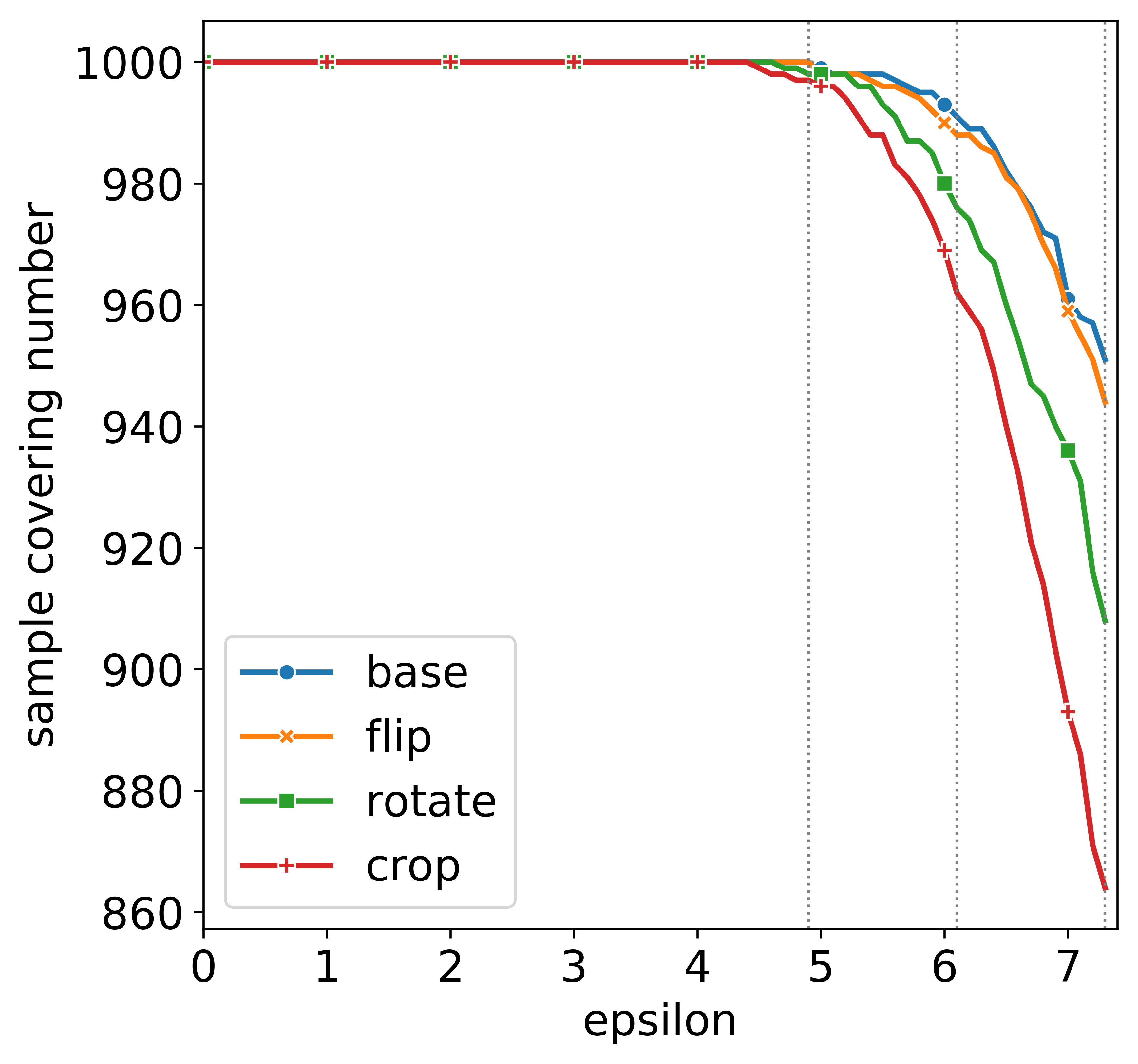}}
    \hfill
    \subfigure[ShapeNet]{\includegraphics[width=0.4\textwidth]{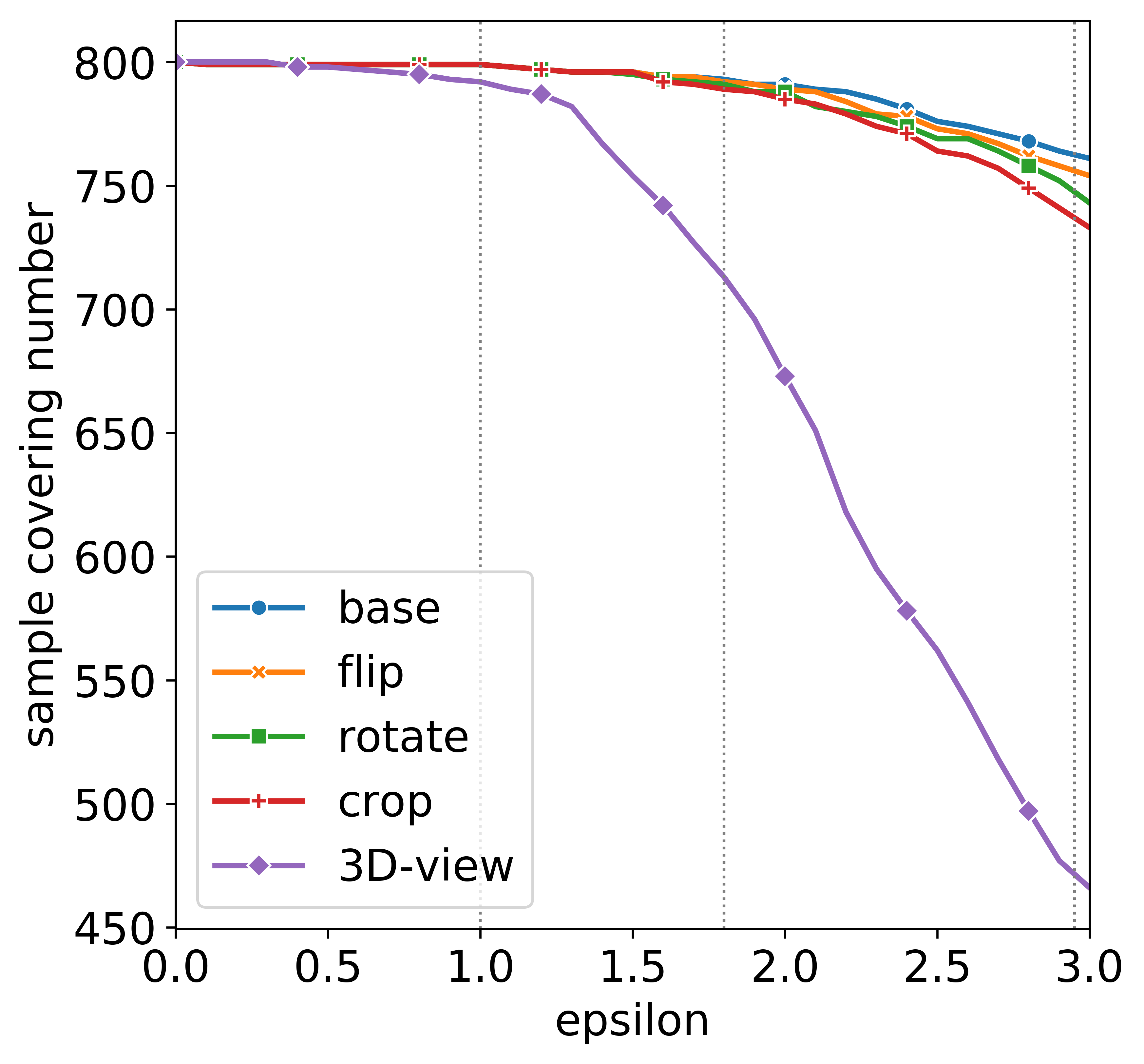}}
    \caption{Estimated sample covering numbers induced by different data transformations at different resolutions $\epsilon$. ``base'' indicates no transformation.
    Note that as $\epsilon$ increases, it starts to exceed the $L_2$ distance between some images and thus some images get covered by others without doing any transformation.
    Three vertical dashed lines indicate the maximum resolution $\epsilon$ at which the ``base'' yields a certain sample covering number, and from left to right they are $100\%n$, $99\%n$, $95\%n$.}
    \label{fig:covering_number}
\end{figure}

We implement the algorithm in Section~\ref{sec:estimation} to estimate the sample covering number induced by different data transformations.
For efficiency, we randomly sample 1000 training images from CIFAR-10 and randomly sample 800 training images from ShapeNet. Appendix~\ref{app:exp} compares results with different sample sizes.
We use the Euclidean norm for defining the sample cover.
For continuous data transformations, we do uniform random sampling to approximate the orbit of a data point.

Figure \ref{fig:covering_number} illustrates the estimated sample covering numbers induced by different transformations at different resolutions $\epsilon$.
As the resolution $\epsilon$ increases, the sample covering number $N(\epsilon, \cS, \rho_\cG)$ induced by any data transformation starts to decrease, indicating a smaller-sized sample cover needed to cover the entire dataset.
Meanwhile, different transformations behave differently.
On CIFAR-10, cropping induces the smallest sample covering number.
On ShapeNet, 3D-view transformation induces the smallest sample covering number and the gap is significant.
Our propositions suggest that data transformations that induce smaller sample covering numbers tend to bring more generalization benefits for the corresponding invariant models.
Therefore, Figure \ref{fig:covering_number} indicates that
models should generalize well if it is invariant to 3D-view transformation on ShapeNet or to cropping on CIFAR-10.

\subsection{Evaluation of Generalization Benefit}\label{exp:benefit}

\begin{table}
\resizebox{\columnwidth}{!}{
\begin{tabular}{lcccccc}
\toprule
            & \multicolumn{2}{c}{$n=100$} & \multicolumn{2}{c}{$n=1000$} & \multicolumn{2}{c}{$n=all$} \\ \cmidrule{2-7} 
Model       & acc (\%)       & gap        & acc (\%)       & gap         & acc (\%)       & gap        \\ \midrule
Base        & $41.05 \pm 0.52 $ & $58.95 \pm 0.52$ & $68.62\pm0.90$ & $31.38\pm 0.90$ & $85.43\pm0.35$ & $14.57\pm0.35$ \\
Flip        & $44.19 \pm 0.74$  & $55.81\pm0.74$ & $75.12\pm0.20$ & $24.88\pm0.20$  & $89.67\pm0.24$ & $10.33\pm0.24$ \\
Rotate      & $47.02 \pm 0.46$  & $52.93\pm0.51$ & $76.07\pm0.28$ & $23.92\pm0.27$  & $89.91\pm0.13$ & $10.05\pm0.16$ \\
Crop        & $\bf{50.47 \pm 0.48}$  & $\mathbf{49.53\pm0.48}$ & $\mathbf{81.84\pm0.12}$ & $\mathbf{18.15\pm0.11}$   & $\mathbf{92.52\pm0.08}$ & $\mathbf{7.48\pm0.08}$ \\ 
\bottomrule
\end{tabular}
}
\caption{Classification accuracy and generalization gap (the difference between training and test accuracy) for ResNet18 on CIFAR-10. $n$ denotes the sample size per class. }
\label{tab:cifar}
\end{table}
\begin{table}
\resizebox{\columnwidth}{!}{
\begin{tabular}{lcccccc}
\toprule
            & \multicolumn{2}{c}{$n=100$} & \multicolumn{2}{c}{$n=1000$} & \multicolumn{2}{c}{$n=all$} \\ \cmidrule{2-7} 
Model       & acc (\%)       & gap        & acc (\%)       & gap         & acc (\%)       & gap        \\ \midrule
Base        & $67.75 \pm 2.02 $ & $32.25\pm2.02$ & $83.33\pm0.38$ & $16.67\pm 0.38$ & $91.81\pm0.22$ & $8.18\pm0.22$ \\
Flip        & $69.75 \pm 1.55$  & $30.25\pm1.55$ & $84.24\pm0.30$ & $15.76\pm0.30$  & $92.07\pm0.20$ & $7.92\pm0.20$ \\
Rotate      & $70.25 \pm 1.19$  & $29.50\pm1.15$ & $83.93\pm0.38$ & $15.94\pm0.35$  & $91.85\pm0.20$ & $8.03\pm0.26$ \\
Crop        & $74.88 \pm 1.03$  & $23.53\pm1.30$ & $86.13\pm0.39$ & $13.75\pm0.32$  & $92.64\pm0.12$ & $7.17\pm0.19$ \\
3D-View     & $\bf{78.13 \pm 1.31}$  & $\mathbf{14.94\pm1.76}$ & $\mathbf{88.79\pm0.34}$ & $\mathbf{8.38\pm0.79}$   & $\mathbf{94.38\pm0.08}$ & $\mathbf{3.09\pm0.10}$ \\ 
\bottomrule
\end{tabular}
}
\caption{Classification accuracy and generalization gap (the difference between training and test accuracy) for ResNet18 on ShapeNet. $n$ denotes the sample size per class. }
\label{tab:shapenet}
\end{table}

We now evaluate the actual generalization performance of invariant models to verify if the sample covering number is a good suitability measurement.
We use ResNet18 \cite{he2016deep} on both datasets and discuss the influence of the model class's implicit bias in Appendix~\ref{app:exp}.
A simple method to learn invariant models is to do data augmentation.
The augmented loss function is $\cL_{aug} (\bx) = \cL(f(g(\bx)))$,  
where $f(\cdot)$ denotes the model and $g(\bx)$ denotes a randomly sampled example in $\bx$'s orbit induced by transformation $\cG$. We use this method on CIFAR-10 and ShapeNet and show results in Table \ref{tab:cifar} and \ref{tab:shapenet}. 

\shortsection{Sample covering number correlates well with generalization benefit} We use the generalization gap (the gap between training accuracy and test accuracy) to measure actual generalization benefit. Compared with the baseline, invariant models show an improved reduced generalization gap and also improved test accuracy. On CIFAR-10, the cropping-invariant model shows the smallest generalization gap and the highest accuracy.
On ShapeNet, the model that is invariant to 3D-view changes shows the smallest generalization gap and the highest accuracy, especially when the training data size is small.
By comparing results in Figure \ref{fig:covering_number} and Table \ref{tab:cifar}-\ref{tab:shapenet}, we observe a clear correlation between the smaller sample covering number and better generalization benefit.
This verifies our proposition --- invariance to more suitable data transformations gives the model more generalization benefit.

\shortsection{Model invariance indeed improves after learning}
\begin{table}
\small
\centering
\begin{tabular}{lccccc}
\toprule
$\lambda$ & train acc (\%) & test acc (\%) & gap & $\cL_{inv}$ & $\cA_{inv}(\%)$ \\
\midrule
$0$         & $99.99\pm 0.01$& $91.81 \pm 0.22$ & $8.19\pm0.22$ & $0.0548\pm0.0028 $ &  $62.0\pm 0.6$      \\
$0.01$       &  $99.98\pm 0.00$& $92.77 \pm 0.16$ & $7.21\pm0.16$ &  $0.0290\pm 0.0029$ & $74.78 \pm 1.61$       \\
$0.1$         &   $99.99\pm 0.00$& $93.87 \pm 0.19$ & $6.11\pm0.19$ &  $0.0152 \pm 0.0003$  & $83.12\pm0.50$      \\
$0.3$     &   $99.98\pm 0.00$& $94.23 \pm 0.11$ & $5.76\pm0.11$ &  $0.0121 \pm 0.0003$  & $85.10\pm0.20$\\
$1$         &   $99.58\pm 0.04$& $94.68 \pm 0.09$ & $4.90\pm0.09$ &   $0.0095 \pm 0.0001$  & $86.94\pm0.08$     \\
$3$      &   $97.74\pm 0.19$& $94.48 \pm 0.19$ & $3.26\pm0.09$ &     $0.0060\pm 0.0003$ & $88.15\pm0.18$     \\
$10$        &    $95.67\pm 0.26$& $93.56 \pm 0.29$ & $2.11\pm0.04$ &   $0.0037 \pm 0.0002$ & $89.20\pm0.16$      \\
$100$       &     $92.89\pm 0.25$& $91.85 \pm 0.26$ & $1.03\pm0.03$ &  $0.0018\pm 0.0001$ &  $89.82\pm0.10$     \\
\bottomrule
\end{tabular}
\caption{Evaluation of ResNet18 on ShapeNet under 3D-view transformations. $\cL_{inv}$ denotes the test invariance loss. $\cA_{inv}$ denotes the test consistency accuracy (indicating whether the model's prediction is unchanged after data transformation) under the worst-case data transformations.}
\label{tab:lambda}
\end{table}
To verify that the improved generalization is indeed brought by model invariance, we further enforce the invariance using the invariance regularization loss similar to \cite{Xie2020,zhang2019theoretically}: $\cL = \cL_{cls}(f(\bx)) + \lambda \text{KL}(f(\bx), f(g(\bx)))$..
Specifically, in addition to minimizing the classification loss on original images, we penalize the model by minimizing the KL divergence between model outputs on original images and on transformed ones. 
At test time, we use $\cL_{inv} (\bx) =\EE_{g_1,g_2\in\cG} [\text{KL}(f(g_1(\bx)), f(g_2(\bx)))]$ to evaluate model invariance under transformation $\cG$.
Table~\ref{tab:lambda} shows that, as we increase the invariance penalty by increasing $\lambda$, invariant models enjoy a smaller generalization gap.
Moreover, the decreased invariance loss and increased consistency accuracy show that model invariance indeed improves after training, demonstrating the generalization benefit brought by model invariance. 

\section{Related Work}\label{sec:related}
\shortsection{Understandings from the input space perspective}
One line of work characterizes the input space of invariant models.
\cite{anselmi_unsupervised_2014,anselmi_invariance_2015} show that the invariant representations equivalently reduce the input dimension for downstream tasks and thus significantly reduce the model complexity (exponential in input dimensions) of downstream linear models.
\cite{sokolic_generalization_2017,sannai_improved_2020} essentially factorize the input space into the product of a base space and a finite set of data transformations. 
Since the covering number needed to cover the base space is smaller, the associated generalization bound for invariant models is reduced.
Compared with these works, our work tries to cover the sample instead of the input space which circumvents the strong dependence on input dimensions and also enables practical evaluation.

\shortsection{Understandings from the function space perspective}
Another line of work directly characterizes the function space of invariant models.
\cite{lyle_benefits_2020} uses PAC-Bayes to show the reduction of generalization upper bound when the model class is symmetrized to be invariant.
\cite{elesedy_provably_2021} analyzes the function space under the feature averaging operator and shows the first strict generalization gap (instead of an upper bound) via a linear model.
This line of work currently restricts model invariance to be obtained exclusively via feature averaging.

Note that the categorization of different understanding perspectives is only for presentation convenience and has no formal distinctions.
Additionally, we mention some work that studies model invariance but does not focus on understanding its benefit.
\cite{abu-mostafa_hints_1993} proves that the VC dimension of an invariant model cannot be larger than its counterpart.
\cite{bloem-reddy_probabilistic_2020} characterizes the general functional representations of invariant probability distributions as well as neural network structures that implement them.
\cite{chen_group-theoretic_2020} uses group theory to show the benefit of learning with data-augmented loss.
In the predicting generalization competition at NeurIPS 2020 \cite{Jiang2020}, the runner-up team \citep{k2021robustness} shows that model robustness to data transformations can serve as an empirical proxy for predicting models' generalization performance.
\cite{robey2021modelbased} enforce model invariance to learned data transformations that capture inter-domain variation to improve the out-of-distribution generalization. 
\citep{Benton_Finzi_Izmailov_Wilson_2020} propose to select data transformations automatically from model training via optimizing parameterized distributions of data transformations. Interestingly, our sample covering numbers may be used to determine their regularization coefficients for better trade-offs.

\section{Conclusion}
This paper investigates the generalization advantage of model invariance by establishing model complexity bounds using the sample cover generated by data transformations. Additionally, we introduce an algorithm to estimate the sample cover and demonstrate that the sample covering number can aid in selecting suitable data transformations through empirical analysis. Our hope is that this research will encourage the exploration of more appropriate data transformations for particular datasets. One potential avenue for future research is to examine the implicit biases of model classes to improve our understanding of the generalization benefits of model invariance.

\section *{Acknowledgements}
This work is supported by a startup fund from the Department of Computer Science of the University of Maryland, National Science Foundation IIS CRII Award, DOD-DARPA-Defense Advanced Research Projects Agency Guaranteeing AI Robustness against Deception (GARD), Air Force Material Command, and Adobe, Capital One and JP Morgan faculty fellowships.

\bibliography{neurips_2021}
\bibliographystyle{plain}

\clearpage
\newpage
\appendix
\appendix
\section{Complexity Measurements and Generalization Bounds} \label{app:gen_bound}
In this section, we provide additional details on complexity measurements and generalization bounds.

The following lemma bounds the empirical Rademacher complexity of a function class $\cH$ via the covering number of $\cH$ evaluated at the sample $\cS$.
\begin{lemma} [Dudley's Entropy Integral Theorem \cite{Dudley_1967,mohri2018foundations}]
\label{lemma:dudley}
Let $\cH$ be a function class from $\cX$ to $\RR$. Then, for any $\alpha>0$,
\begin{align*}
\rR_\cS(\cH) \leq 4\alpha + 12\int_\alpha^\infty \sqrt{\frac{\log N\big(\tau, \cH, L_2(\PP_{\cS})\big)}{n}} d\tau.
\end{align*}
\end{lemma}

The following theorem provides a uniform generalization bound for a function class via empirical Rademacher complexity. 
\begin{theorem}[\cite{Bartlett_Mendelson_2002,mohri2018foundations}]
\label{thm:standard-generalization}
Let $\cH$ be a function class from $\cX$ to $[0,B]$.
For any $\delta>0$, with probability at least $1-\delta$ over the draw of a sample $\cS$ with size $n$ according to data distribution $\cD$, the following holds for any $h\in\cH$:
\begin{align}
\label{eqn:generalizaion-bound}
  R(h) &\leq R_{\cS}(h) + 2B\rR_{\cS}(\cH) + 3B\sqrt{\frac{\log\frac{2}{\delta}}{2n}}
\end{align}
\end{theorem}
We can plug the refined Rademacher complexity bounds in Proposition~\ref{prop:lipschitz-refined-Ramemacher} and Theorem~\ref{thm:linear} into \eqref{eqn:generalizaion-bound} to get refined generalization bounds for certain invariant models.

\section{Proofs}\label{sec:app:proofs}
We first prove Theorem~\ref{thm:lipschitz}, and then Proposition~\ref{prop:upper_bound}.

\subsection{Proof of Theorem~\ref{thm:lipschitz}} \label{subsec:proof-thm-3.4}
\begin{proof}[Proof of Theorem~\ref{thm:lipschitz}]

The general idea of this proof is to show that any cover of  \textit{a model class evaluated at a sample cover} also generates a same-sized cover of \textit{the model class evaluated at the original sample} with some additional approximation error. 
The covering number inequality in \eqref{eqn:thm-lipschitz} then follows by taking the minimization over all covers of the model class evaluated at the original sample.
Since this proof includes some tedious notations, we first restate the problem setup and then go to the details.

\shortsection{Problem setup}
Let $\cS=\{\bx_1, \bx_2, \ldots, \bx_n\}$ be a sample of size $n$.
Let $\hat{\cS}\subseteq \cS$ be an $\epsilon$-cover of $\cS$ with respect to $\rho_\cG$ and has size $m$.
Without loss of generality, we then vectorize $\cS$ and $\hat{\cS}$ for notation simplicity.
Denote by $S=(\bx_1, \bx_2, ..., \bx_n)^T$ the vectorized sample associated with $\cS$ in some arbitrary but fixed order. 
Denote by $\hat{S}=(\hat{\bx}_1, \hat{\bx}_2, ..., \hat{\bx}_m)^T$ the vectorized sample cover associated with $\hat{\cS}$ in some arbitrary but fixed order.
$S$ and $\hat{S}$ thus define a matrix $\Pb$ below indicating how $\hat{S}$ approximately recovers $S$:
$$ \Pb=(p_{ij})\in\RR^{n\times m} \quad\text{such that}\quad p_{ij}=\left\{ \begin{array}{cc}
    1, & \text{if $\bx_i\in\cS$ is approximated by $\hat{\bx}_j\in\hat{\cS}$} \\
    0, & \text{otherwise}
\end{array} \right. .$$
We use arbitrary tie-breaking rule when a data point $\bx\in\cS$ can be approximated by multiple $\hat{\bx}\in\hat{\cS}$.
Without loss of generality, we also assume that there is no "redundant" element in $\hat{\cS}$ which is not used in recovering $\cS$, since otherwise it can be removed from $\hat{\cS}$ for a strictly smaller cardinality.
Therefore, by definition, $\Pb$ has linearly independent columns and thus represents an injective mapping from $\RR^m$ to $\RR^n$.
We denote by $S'$ the approximately recovered $S$ generated by $\hat{S}$: $S'=\Pb \hat{S}$.
The first line in Table~\ref{tab:diagram-of-proof} shows the relationship among $\hat{S}$, $\hat{S'}$, and $S$. 

\begin{table}
\centering
\begin{tabular}{ccccc}
$\hat{S}$ & $\xrightarrow[]{\quad\text{generates}\quad}$ & $S'=P\hat{S}$ & $\xrightarrow[]{\quad\text{approximates}\quad}$ & $S$ \\ \\
$\cT(\cH_{|\hat{S}})$ & $\xrightarrow[\text{\textbf{step (I)}}]{\quad\text{generates}\quad}$ & $\cT(\cH_{|S'})$ & $\xrightarrow[\text{\textbf{step (II)}}]{\quad\text{is also}\quad}$ &  $\cT(\cH_{|S})$
\end{tabular}
\caption{A diagram of the proof of Theorem~\ref{thm:lipschitz}.}
\label{tab:diagram-of-proof}
\end{table}

\begin{table}
\centering
\begin{tabular}{lll}
\toprule
\textbf{Space} & \textbf{Vector} & \textbf{Vector (in the cover)} \\
\midrule
$(\RR^m, \rho_m)\quad$ & $h_{|\hat{S}}\in\cH_{|\hat{S}}\quad$ & $\hat{h}_{|\hat{S}}\in\cT(\cH_{|\hat{S}})$ \\
$(\RR^n, \rho_n)\quad$ & $h_{|S'}\in\cH_{|S'}\quad$ & $\hat{h}_{|S'}\in\cT(\cH_{|S'})$ \\
$(\RR^n, \rho_n)\quad$ & $h_{|S}\in\cH_{|S}\quad$ & $\hat{h}_{|S}\in\cT(\cH_{|S})$ \\
\bottomrule
\end{tabular}
\caption{Some notations used in the proof of Theorem~\ref{thm:lipschitz}.}
\label{tab:appendix-proof-notations}
\end{table}

Based on the definition of $\Pb$, we now give the precise definition of $p(\bx)$ used in defining the (pseudo)metrics in this theorem. 
$$ p(\hat{\bx}_j) = \sum_{i=1}^n p_{ij}, \quad \forall j\in[m]. $$

We proceed to introduce notations for the model class.
Instead of considering the model class $\cH$ under the metric induced by the function norm $L_2(\PP_\cS)$ (or $L_2(\PP_{\hat{\cS}})$), we equivalently consider the evaluation of $\cH$ at $S$ (or $\hat{S}$) under the metric $\rho_n$ (or $\rho_m$) in this proof for notation simplicity.
We denote the evaluation of $\cH$ at $S$ as $\cH_{|S} = \{ (h(\bx_1),\ldots,h(\bx_n))^T: h\in\cH \}$, and similarly its evaluation at $\hat{S}$ as $\cH_{|\hat{S}} = \{ (h(\hat{\bx}_1),\ldots,h(\hat{\bx}_m))^T: h\in\cH \}$.
We define the metric $\rho_n$ on $\RR^n$ as $\rho_n(u, u')=\frac{1}{\sqrt{n}}\|u-u'\|_2$, and the metric $\rho_m$ on $\RR^m$ as $\rho_m(v, v')=\frac{1}{\sqrt{n}}\|(\Pb^T\Pb)^{\frac{1}{2}}(v-v')\|_2$.
Therefore, the covering number notation $N\big(\tau, \cH, L_2(\PP_\cS)\big)$ is equivalent to $N(\tau, \cH_{|S}, \rho_n)$, and $N\big(\tau, \cH, L_2(\PP_{\hat{\cS}_{\cG,\epsilon}})\big)$ is equivalent to $N(\tau, \cH_{|\hat{S}}, \rho_m)$. Table~\ref{tab:appendix-proof-notations} shows an overview of these notations.

\shortsection{Summary}
The proof has the following steps.
\textbf{(I)} Any cover $\cT(\cH_{|\hat{S}})$ of a model class evaluated at the sample cover $\hat{S}$ generates a same-sized cover $\cT(\cH_{|S'})$ of the model class evaluated at the approximated sample $S'$.
\textbf{(II)} The cover $\cT(\cH_{|S'})$ of the model class evaluated at the approximated sample is also a cover $\cT(\cH_{|S})$ of the model class evaluated at the original sample $S$.
\textbf{(III)} The covering number inequality follows by taking the minimization over all covers of the model class evaluated at the original sample $S$.

\shortsection{Step (I)}
We first show that any cover $\cT(\cH_{|\hat{S}})$ of $\cH_{|\hat{S}}$ generates a set, denoted as $\cT(\cH_{|S'})$, with the same cardinality. Given any $\cT(\cH_{|\hat{S}})$, we construct $\cT(\cH_{|S'})=\{ \Pb\hat{h}_{|\hat{S}}: \hat{h}_{|\hat{S}}\in \cT(\cH_{|\hat{S}}) \}$. Since $\Pb$ represents injective mapping from $\RR^m$ to $\RR^n$, we have $|\cT(\cH_{|S'})|=|\cT(\cH_{|\hat{S}})|$ by construction.

Then, we show that $\cT(\cH_{|S'})$ is a $\tau$-cover of $\cH_{|S'}$ with respect to $\rho_n$ if $\cT(\cH_{|\hat{S}})$ is a $\tau$-cover of $\cH_{|\hat{S}}$ with respect to $\rho_m$.
By the definition of $\Pb$, it holds that $h_{|S'}=\Pb h_{|\hat{S}}$ for any $h\in\cH$, and $\Pb^\dagger\Pb=\Ib$ where $\Pb^\dagger$ is the Moore–Penrose inverse of $\Pb$ since $\Pb$ has linearly independent columns. 
Thus, for any $h_{|S'}\in\cH_{|S'}$, we can project it to $\cH_{|\hat{S}}$ by $\Pb^\dagger h_{|S'}$.
Given that $\cT(\cH_{|\hat{S}})$ is a $\tau$-cover of $\cH_{|\hat{S}}$ with respect to $\rho_m$, for any $h_{|S'}\in\cH_{|S'}$, there exists $\hat{h}_{|\hat{S}} \in \cT(\cH_{|\hat{S}})$ such that $\rho_m(\Pb^\dagger h_{|S'}, \hat{h}_{|\hat{S}})\le \tau$. It follows that
\begin{align*}
\rho_m(\Pb^\dagger h_{|S'}, \hat{h}_{|\hat{S}}) &= \frac{1}{\sqrt{n}}\| (\Pb^T\Pb)^{\frac{1}{2}} (\Pb^\dagger h_{|S'} - \hat{h}_{|\hat{S}}) \|_2 \\
&= \frac{1}{\sqrt{n}} \sqrt{ (\Pb^\dagger h_{|S'} - \hat{h}_{|\hat{S}})^T (\Pb^T\Pb) (\Pb^\dagger h_{|S'} - \hat{h}_{|\hat{S}}) } \\
&= \frac{1}{\sqrt{n}} \sqrt{ (h_{|S'} - \Pb\hat{h}_{|\hat{S}})^T (h_{|S'} - \Pb\hat{h}_{|\hat{S}}) } \\
&= \frac{1}{\sqrt{n}}\|  (h_{|S'} - \Pb\hat{h}_{|\hat{S}} )\|_2 \\
&= \rho_n(h_{|S'}, \hat{h}_{|S'}) \le \tau,
\end{align*}
where $\Pb\Pb^\dagger h_{|S'}=\Pb\Pb^\dagger\Pb h_{|\hat{S}}=\Pb h_{|\hat{S}}=h_{|S'}$, and $\hat{h}_{|S'}=\Pb\hat{h}_{|\hat{S}}$ is in $\cT(\cH_{|S'})$ by construction and approximates the given $h_{|S'}$.
Therefore, for any $h_{|S'}\in\cH_{|S'}$, there exists $\hat{h}_{|S'}\in \cT(\cH_{|S'})$ such that $\rho_n(h_{|S'}, \hat{h}_{|S'})\leq \tau$, which implies that $\cT(\cH_{|S'})$ is a $\tau$-cover of $\cH_{|S'}$.

\shortsection{Step (II)}
We proceed to show that $\cT(\cH_{|S'})$ is also a $(\tau+\kappa\epsilon\sqrt{1-\frac{|\hat{\cS}|}{n}})$-cover of $\cH_{|S}$.
Consider any index $i\in[n]$.
Given that $\hat{\cS}$ is an $\epsilon$-sample cover of $\cS$ with respect to $\rho_\cG$, we have $\rho_\cG(\bx_i, \bx'_i) = \inf_{\gamma\in\Gamma(\bx_i, \bx'_i)} \int_\gamma c(r) dr \le \epsilon$. Moreover, for any $\xi>0$, by the definition of infimum there exists a path $\gamma_0$ such that $\int_{\gamma_0} c(r) dr \le \epsilon + \xi$. 
The following result then shows that the evaluations of any $h\in\cH$ at data points $\bx_i$ and $\bx'_i$ are close (let $\nabla_{\bx} h\in\partial h(\bx)$ when $h$ is only subdifferentiable at some $\bx$):
\begin{align*}
|h(\bx_i)-h(\bx'_i)| =& \int_{\gamma_0} \nabla_{\bx} h(\br) \cdot d\br \\
\le& \int_{\gamma_0} \|\nabla_{\bx} h(\br)\|\, ds \tag{$ds = \|d\br\|$} \\
=& \int_{\gamma_0} \|\nabla_{\bx} h(\br)\|\, c(\br) ds \tag{invariance of $h$}\\
\le&  \kappa \int_{\gamma_0} c(\br) ds \tag{Lipschitzness of $h$}\\
=& \kappa(\epsilon+\xi).
\end{align*}
Since it holds for any $\xi>0$, we have $|h(\bx_i)-h(\bx'_i)| \le \kappa\epsilon$.

Thus, the evaluations of any $h\in\cH$ at samples $S$ and $S'$ are close with respect to $\rho_n$:
\begin{align*}
\frac{1}{\sqrt{n}} \| h_{|S} - h_{|S'} \|_2 = \frac{1}{\sqrt{n}} \sqrt{\sum_{i=1}^n \left(h(\bx_i)-h(\hat{\bx}_i)\right)^2} 
\leq \frac{1}{\sqrt{n}} \sqrt{(\kappa\epsilon)^2 (n-|\hat{\cS}|)}
= \kappa\epsilon\sqrt{1-\frac{|\hat{\cS}|}{n}}.
\end{align*}

Therefore, given any $h_{|S}\in \cH_{|S}$, we have $h_{|S'}\in \cH_{|S'}$ such that $\rho_n(h_{|S}, h_{|S'}) \le \kappa\epsilon\sqrt{1-\frac{|\hat{\cS}|}{n}}$ and we can find $\hat{h}_{|S'} \in \cT(\cH_{|S'})$ such that $\rho_n(h_{|S'}, \hat{h}_{|S'}) \le \tau$ since $\cT(\cH_{|S'})$ is an $\tau$-cover of $\cH_{|S'}$. It then follows that $\hat{h}_{|S'}$ approximates $h_{|S}$:
$$ \rho_n(h_{|S}, \hat{h}_{|S'}) \le \rho_n(h_{|S}, h_{|S'}) + \rho_n(h_{|S'}, \hat{h}_{|S'}) \le \tau + \kappa\epsilon\sqrt{1-\frac{|\hat{\cS}|}{n}}, $$
which implies that $\cT(\cH_{|S'})$ is a $(\tau+\kappa\epsilon\sqrt{1-\frac{|\hat{\cS}|}{n}})$-cover of $\cH_{|S}$.

\shortsection{Step (III)}
The final covering number inequality proceeds as follows.
Note that any $\tau$-cover $\cT(\cH_{|\hat{S}})$ of $\cH_{|\hat{S}}$ generates an $(\tau+\kappa\epsilon\sqrt{1-\frac{|\hat{\cS}|}{n}})$-cover $\cT(\cH_{|S})$ of $\cH_{|S}$ such that $|\cT(\cH_{|\hat{S}})| = |\cT(\cH_{|S})|$. 
The set of all covers of $\cH_{|\hat{S}}$ then generates a set of covers of $\cH_{|S}$, which further constitutes a subset of all covers of $\cH_{|S}$. Thus, we have the following covering number inequality:
\begin{align*}
& N(\tau+\kappa\epsilon\sqrt{1-\frac{|\hat{\cS}|}{n}}, \cH_{|S}, \rho_n) \\
=& \min\{ |\cT(\cH_{|S})|: \cT(\cH_{|S}) \text{ is a cover of } \cH_{|S} \} \\
\le& \min\{ |\cT(\cH_{|S})|: \cT(\cH_{|S}) \text{ is a cover of } \cH_{|S} \text{ and is generated by } \cT(\cH_{|\hat{S}}) \} \\
=& \min\{ |\cT(\cH_{|\hat{S}})|: \cT(\cH_{|\hat{S}}) \text{ is a cover of } \cH_{|\hat{S}} \} \\
=& N(\tau, \cH_{|\hat{S}}, \rho_m).
\end{align*} 

Since this inequality holds for any resolution $\epsilon$ and any $\epsilon$-sample cover, taking the infimum over all resolutions and sample covers and replacing variables then yields the inequality in \eqref{eqn:thm-lipschitz}.

\end{proof}

\subsection{Proof of Proposition~\ref{prop:upper_bound}}
We first provide a lemma for proving Proposition~\ref{prop:upper_bound}.
Note that theorem~\ref{thm:lipschitz} directly leads to a refined empirical Rademacher complexity bound in terms of the covering number of $\cH$ evaluated at the sample cover. The following lemma is a weaker but simpler version. We can set $\epsilon\rightarrow0$ as $n\rightarrow\infty$ to further suppress the additional error term on large samples.

\begin{lemma}[Refined Rademacher complexity of $\cG$-invariant $\cH$]\label{prop:lipschitz-refined-Ramemacher}
Let $\cS=\{\bx_i\}_{i=1}^n$ be a sample of size $n$.
Let $\cH$ be a model class such that every $h\in\cH$ is $\kappa$-Lipschitz with respect to $\|\cdot\|$ and invariant to $\cG$. 
Given any $\epsilon>0$, $\alpha>0$, let ${\hat{\cS}_{\cG,\epsilon}}$ be an $\epsilon$-cover of $\cS$.
Then
\begin{align}
\label{eqn:dudley}
\rR_\cS(\cH) \leq 4\kappa\epsilon \sqrt{1-\frac{|{\hat{\cS}_{\cG,\epsilon}}|}{n}} + 4\alpha + 12\int_\alpha^\infty \sqrt{\frac{\log N\big(\tau, \cH, L_2(\PP_{\hat{\cS}_{\cG,\epsilon}})\big)}{n}} d\tau.
\end{align}
\end{lemma}

\begin{proof}
Given any $\alpha>0$, let $\alpha'=\alpha + \kappa\epsilon \sqrt{1-\frac{|{\hat{\cS}_{\cG,\epsilon}}|}{n}}$ and $\tau'=\tau + \kappa\epsilon \sqrt{1-\frac{|{\hat{\cS}_{\cG,\epsilon}}|}{n}}$. Plugging \eqref{eqn:thm-lipschitz} into Dudley's entropy integral theorem (Lemma~\ref{lemma:dudley}) yields
\begin{align*}
\rR_\cS(\cH) &\le 4\alpha' + 12\int_{\alpha'}^\infty \sqrt{\frac{\log N\big(\tau', \cH, L_2(\PP_{\cS})\big)}{n}} d\tau' \\
&\le 4\alpha' + 12\int_{\alpha'}^\infty \sqrt{\frac{\log N\big(\tau' - \kappa\epsilon \sqrt{1-\frac{|{\hat{\cS}_{\cG,\epsilon}}|}{n}}, \cH, L_2(\PP_{\hat{\cS}_{\cG,\epsilon}})\big)}{n}} d\tau' \\
&= 4\alpha' + 12\int_{\alpha' - \kappa\epsilon \sqrt{1-\frac{|{\hat{\cS}_{\cG,\epsilon}}|}{n}}}^\infty \sqrt{\frac{\log N\big(\tau, \cH, L_2(\PP_{\hat{\cS}_{\cG,\epsilon}})\big)}{n}} d\tau' \\
&= 4\kappa\epsilon \sqrt{1-\frac{|{\hat{\cS}_{\cG,\epsilon}}|}{n}} + 4\alpha + 12\int_\alpha^\infty \sqrt{\frac{\log N\big(\tau, \cH, L_2(\PP_{\hat{\cS}_{\cG,\epsilon}})\big)}{n}} d\tau.
\end{align*}
\end{proof}

\begin{proof}[Proof of Proposition~\ref{prop:upper_bound}]
Let $\cH$ be an invariant model class mapping from $\cX$ to $[-B, B]$ for some $B>0$.
Let $\hat{\cS}_{\cG,0}$ be a sample cover such that $|\hat{\cS}_{\cG,0}|=N(0,\cS,\rho_\cG)=m$.

We construct a $\tau$-cover of $\cH$ with respect to $L_2(\PP_{\hat{\cS}_{\cG,0}})$ as follows:
for every $\bx\in\hat{\cS}_{\cG,0}$, we cover the output range of $\cH$ at $\bx$ by a set of points 
$$\cY = \{-B + (2k-1)\tau\}_{k=1}^{\lceil{\frac{B}{\tau}}\rceil}.$$
Let $\hat{\cH}=\{ \hat{h}: \text{dom}(h)=\hat{\cS}_{\cG,0}, \hat{h}(\bx)\in\cY, \forall \bx\in\hat{\cS}_{\cG,0}\}$.
To see that $\hat{\cH}$ is indeed a $\tau$-cover of $\cH$ with respect to $L_2(\PP_{\hat{\cS}_{\cG,0}})$, given any $h\in\cH$, we choose $\hat{h}\in\hat{\cH}$ such that $|h(\bx)-\hat{h}(\bx)|\le \tau$ for every $\bx\in\hat{\cS}_{\cG,0}$ and thus
\begin{align*}
\|h-\hat{h}\|_{L_2(\PP_{\hat{\cS}_{\cG,0}})} &= \left(\sum_{\bx\in{\hat{\cS}_{\cG,0}}} \frac{p(\bx)}{n} \big(h(\bx)-\hat{h}(\bx)\big)^2\right)^\frac{1}{2} \\
&\le \left(\frac{\sum_{\bx\in{\hat{\cS}_{\cG,0}}}p(\bx)}{n} \tau^2\right)^\frac{1}{2} \\
&\le \tau
\end{align*}
Therefore, for $\tau<B$, the covering number of $\cH$ satisfy
$$ N\big(\tau, \cH, L_2(\PP_{\hat{\cS}_{\cG,0}})\big) \le \lceil{\frac{B}{\tau}}\rceil ^m \le \left(\frac{B}{\tau} + 1\right)^m \le \left(\frac{2B}{\tau}\right)^m, $$
whereas for $\tau>B$, we have $N\big(\tau, \cH, L_2(\PP_{\hat{\cS}_{\cG,0}})\big) \le \lceil{\frac{B}{\tau}}\rceil^m \le 1$.

Note that Proposition~\ref{prop:lipschitz-refined-Ramemacher} holds for any model class if we set $\epsilon=0$. Plugging $N\big(\tau, \cH, L_2(\PP_{\hat{\cS}_{\cG,0}})\big)$ into Proposition~\ref{prop:lipschitz-refined-Ramemacher} and setting $\epsilon=0$, $\alpha=0$, we have
\begin{align*}
\rR_\cS(\cH) &\le 12\int_0^\infty \sqrt{\frac{\log N\big(\tau, \cH, L_2(\PP_{\hat{\cS}_{\cG,0}})\big)}{n}} d\tau \\
&\le 12\int_0^B \sqrt{\frac{m \log \left(\frac{2B}{\tau}\right)}{n}} d\tau \\
&= 24B\sqrt{\frac{m}{n}} \int_0^\frac{1}{2} \sqrt{\log \left(\frac{1}{t}\right)} dt \\
&\le 24B\sqrt{\frac{m}{n}} \int_0^1 \sqrt{\log \left(\frac{1}{t}\right)} dt \\
&= 24B\sqrt{\frac{m}{n}} \cdot \frac{\sqrt{\pi}}{2} \\
&\le 24B\sqrt{\frac{m}{n}}
\end{align*}
\end{proof}

\subsection{Binary Coding Constructions of Data Transformations in Proposition~\ref{prop:uniform-bound}}\label{sec:app:dt-construction}
In Proposition~\ref{prop:uniform-bound}, given $K$ sets of group-structured data transformations $\{\cG^{(1)}, \cG^{(2)},... , \cG^{(K)}\}$, we provide a uniform bound for any $h$ in model class and any set of data transformations. 
Here, we extend it to any set of combinatorial data transformations. Given a pool of $L$ types of group-structured data transformations $\{\cG^{(1)}, \cG^{(2)},..., \cG^{(L)}\}$ (e.g., rotation, flipping), we construct the combinatorial data transformations selection $\cG_k$ indexed by $k$ as follows:
fix an arbitrary order of the power set of $[L]$ and denote the $k$-th element as $\cI_k$.
For any $k\in[2^L]$, let $\cG_k$ be the direct product of the data transformations selected by $\cI_k$: $\cG_k = \Pi_{i \in \cI_k} \cG^{(i)}$. Note that $\cG_k$ is also group-structured since the direct product preserves the group structure. Proposition~\ref{prop:uniform-bound} also applies to these combinatorial data transformations $\{\cG_k\}_{k=1}^{2^L}$.

\section{Refined Complexity Analysis for Linear Models}\label{subsec:linear}

This subsection shows a more interpretable generalization benefit of model invariance by considering linear model class and linear data transformations (e.g., rotation).
The following theorem provides a refined model complexity result for the invariant Linear model class.

\begin{theorem} [Refined Rademacher complexity of $\Ab$-invariant $\cH^{\mathsf{Linear}}$]
\label{thm:linear}
Let $\cS=\{\bx_i\}_{i=1}^n$ be a sample of size $n$.
Let $\Ab$ be the matrix representation of any linear data transformation. Consider the $L_p$-norm-bounded linear model class $\cH=\{\bx\mapsto\langle\bw,\bx\rangle:\bw\in\RR^d, \|\bw\|_p \leq W\}$ for some $p\geq1$ and constant $W>0$. Let $\cH^{\mathsf{Linear}}=\{h\in\cH: h(\bx)=h(\Ab\bx), \forall{\bx\in\RR^d}\}$ be the subset of $\cH$ that is invariant under transformation $\Ab$. Then
\begin{align}
\label{eq:linear}
\rR_\cS(\cH^{\mathsf{Linear}}) =& \frac{W}{n} E_\sigma\left[\inf_{\bm{\eta}\in\RR^d}\left\|\bu_\sigma + (\Ab-\Ib)\bm{\eta}\right\|_q\right],
\end{align}
where $\bu_\sigma=\sum_{i=1}^n\sigma_i\bx_i$ and $\{\sigma_1, \ldots, \sigma_n\}$ are i.i.d. Rademacher random variables.
\end{theorem}

\begin{proof}
The linearity of the model class $\cH$ allows us to translate the model invariance to an explicit model class constraint and then precisely compute the Rademacher complexity.

To see that the model invariance, $\langle \bw, \bx\rangle = \langle \bw, \Ab\bx\rangle$ for all $\bx\in\RR^d$, is equivalent to an explicit model class constraint $\bw=\Ab^T\bw$, we can choose $\bx$ to be elements in the standard basis of $\RR^d$ and conclude that corresponding entries in $\bw$ and $\Ab^T\bw$ are equal.

Then we precisely compute the Rademacher complexity of $\cH$. Let $\bu_\sigma=\sum_{i=1}^n\sigma_i\bx_i$, we have
\begin{align*}
\rR_S(\cH') =& \E_\sigma\left[\sup_{\substack{\|\bw\|_p\leq W\\(\Ab^T-\Ib)\bw=\textbf{0}}} \frac{1}{n}\sum_{i=1}^n\sigma_i\langle\bw,\bx_i\rangle\right]\\
=& \frac{1}{n} \E_\sigma\left[\sup_{\substack{\|\bw\|_p\leq W\\(\Ab^T-\Ib)\bw=\textbf{0}}} \langle\bw,\bu_\sigma\rangle\right]\\
=& \frac{1}{n} \E_\sigma\left[\sup_{\|\bw\|_p\leq W}\inf_{\bm{\eta}\in\RR^d} \langle\bw,\bu_\sigma\rangle + \langle\bw, (\Ab-\Ib)\bm{\eta}\rangle \right]\\
=& \frac{1}{n} \E_\sigma\left[\inf_{\bm{\eta}\in\RR^d}\sup_{\|\bw\|_p\leq W} \langle\bw,\bu_\sigma + (\Ab-\Ib)\bm{\eta}\rangle \right] \tag{$\star$}\\
=& \frac{W}{n} E_\sigma\left[\inf_{\bm{\eta}\in\RR^d}\left\|\bu_\sigma + (\Ab-\Ib)\bm{\eta}\right\|_q\right], \tag{Dual norm}
\end{align*}
where the equality in ($\star$) holds by the von Neumann-Fan minimax theorem, since $\{\bm{\eta}: \bm{\eta}\in\RR^d\}$ is convex, $\{\bw: \|\bw\|_p\leq W\}$ is compact and convex, and $\langle\bw,\bu_\sigma + (\Ab-\Ib)\bm{\eta}\rangle$ is bi-linear in $\bw$ and $\bm{\eta}$.
\end{proof}

\begin{remark}
As a comparison, the Rademacher complexity of the general linear model class $\cH$ is $\rR_\cS(\cH) = \frac{W}{n} E_\sigma\left[\left\|\bu_\sigma\right\|_q\right]$. 
Note that we always have the model complexity gap $\rR_\cS(\cH) - \rR_\cS(\cH^{\mathsf{Linear}}) \ge 0$ in Theorem~\ref{thm:linear} (as one can check by taking $\bm{\eta}=\bm{0}$ in \eqref{eq:linear}) and the gap can also be made strict in many cases.
\end{remark}
The following proposition gives a more interpretable result by further considering the $L_2$-norm-bounded linear model class.

\begin{proposition}[Refined Rademacher complexity of $L_2$-norm-bounded $\Ab$-invariant $\cH^{\mathsf{Linear}}$]
\label{prop:linear-bounded}
Let $\cH^{\mathsf{Linear}}$ be the $L_2$-norm-bounded linear model class that is invariant under transformation $\Ab$ for some constant $W>0$ (i.e., $p=2$ in Theorem~\ref{thm:linear}). Then
\begin{align}
\label{eqn:l2-bounded}
\rR_\cS(\cH^{\mathsf{Linear}}) = \frac{W}{n} E_\sigma\left[\left\| \Pb \bu_\sigma \right\|_2\right],
\end{align}
where $\Pb=\Ib - (\Ab-\Ib)(\Ab-\Ib)^\dagger$ and $(\Ab-\Ib)^\dagger$ is the Moore–Penrose inverse of $\Ab-\Ib$. 
\end{proposition}

\begin{proof}
Proposition~\ref{prop:linear-bounded} follows from the least square solution to Theorem~\ref{thm:linear} (with $p=2$).
\end{proof}

\begin{remark}
Proposition~\ref{prop:linear-bounded} shows that the improvement in model complexity (and thus the generalization bound) for linear invariant models depends both on the sample and on data transformations.
The matrix $\Pb$ in \eqref{eqn:l2-bounded} is essentially the orthogonal projection matrix that projects the weighted sum of data $\bu_\sigma$ onto the null space of $(\Ab-\Ib)^T$.
Intuitively, the linear data transformation $\Ab$ separates each input $\bx$ into two orthogonal components: $\Pb \bx$ that is $\Ab$-invariant, and $\bx-\Pb \bx$ that is $\Ab$-variant. Linear models that are invariant to $\Ab$ will ignore the $\Ab$-variant component and only capture the $\Ab$-invariant component (otherwise they will not be $\Ab$-invariant).
Suppose that the data distribution has zero mean and bounded variance, then the Rademacher complexity of $\cH^{\mathsf{Linear}}$ is upper-bounded by the variance of the $\Ab$-invariant component in $\bx$. Therefore, if the data transformation captures most of the data variance, the corresponding invariant models will have much smaller model complexity and thus better generalization performance.
We give some examples in Example~\ref{example:linear}.
\end{remark}

\begin{example}
\label{example:linear}
Suppose the data $\bx\in\RR^d$ have Gaussian distribution $\cN(0, \sigma^2\Ib)$. Let $\cH$ be the $L_2$-norm-bounded linear model class. Then we have the following Rademacher complexity \cite{mohri2018foundations} bounds:

(a) $\rR_n(\cH) \leq \sqrt{d}\cdot\frac{W\sigma}{\sqrt{n}}$ for the general $\cH$; 

(b) $\rR_n(\cH') \leq \sqrt{\lceil{\frac{d}{2}}\rceil}\cdot\frac{W\sigma}{2\sqrt{n}}$ for the flipping-invariant $\cH'\subseteq\cH$; 

(c) $\rR_n(\cH'') \leq 1\cdot\frac{W\sigma}{n}$ for the circular-translation-invariant $\cH''\subseteq\cH$. The fast convergence rate of $O(\frac{1}{n})$ guarantees a small generalization gap.
\end{example}

\section{Empirical Estimation of Sample Covering Numbers} \label{app:estimate_sample_cover_num}
Detailed steps to estimate sample covering numbers are as follows.

\shortsection{Step 1}
Compute (or approximate) the direct orbit distance between any two examples in $\cS$.
The direct orbit distance between any two examples $\bx_i, \bx_j \in \cS$ is
\begin{align*}
    d_{\cG}(\bx_i, \bx_j) = \|{\cG}(\bx_i) - {\cG}(\bx_j)\| = \min_{g_1, g_2 \in \cG} \|g_1(\bx_i)-g_2(\bx_j)\|.
\end{align*}

\shortsection{Step 2}
Compute the $\rho_\cG$ distance between any two examples in $\cS$.
Given results in step 1, Computing the $\rho_\cG$ distance between any two examples can be formulated as a shortest path problem on a complete graph, where each node represents an example and the cost of each edge is the direct orbit distance computed in step 1. The shortest path problem is as follows.

\begin{align*}
    \rho_\cG(\bx_s, \bx_t) = \min & \sum_{(i,j)\in [|S|]}  d_{\cG}(\bx_i, \bx_j)z_{ij} \\
    \text{s.t.} & \sum_{j\in \delta^+(i)}z_{ij} - \sum_{j\in \delta^-(i)}z_{ji} = 
    \begin{cases}
      1, & \text{if}\ i=s \\
      -1, & \text{if}\ i=t \\
      0, & \text{o.w.}
    \end{cases}
    ,\quad \forall i\in[|S|] \\
     & \sum_{j\in \delta^+(i)}z_{ij} \leq 1  ,\quad \forall i\in[|S|] \\
    & z_{ij} \in \{0,1\}  ,\quad \forall i,j\in[|S|]
\end{align*}

where $z_{ij}$ is the binary variable indicating whether the path from ${\cG}(\bx_i)$ to ${\cG}(\bx_j)$ belongs to the shortest path, and $\delta^+(i)$, $\delta^-(i)$ are the sets of indices of outgoing and incoming nodes. 
For each pair of examples, this problem can be solved by shortest path algorithms (e.g., Dijkstra's algorithm) in polynomial time (e.g., $O(n^3$)).

\shortsection{Step 3}
Construct the pairwise distance matrix $\bm{\mathsf{D}}\gets [\rho_\cG(\bx_i, \bx_j)]_{i,j}$ and approximate the sample covering number.
In experiments, we use modified k-medoids \cite{park2009simple} clustering method to find the approximation of $N(\epsilon, \cS, \rho_\cG)$. Since the k-medoids algorithm requires the number of clusters as an input, we can assign one heuristically or greedy search it as in Algorithm \ref{alg:cluster}.   
\begin{algorithm}
\caption{\texttt{Distance2SampleCoveringNum}: sample covering number approximation based on pairwise distances
}\label{alg:cluster}
{\bf Input:} distance matrix $\bm{\mathsf{D}}\in \RR^{|\cS|\times|\cS|}$, resolution $\epsilon$

{\bf Output:} $\widehat{N}(\epsilon, \cS, \rho_\cG)$, an empirical estimation of sample covering number $N(\epsilon, \cS, \rho_\cG)$

{\bf Algorithm:}
\begin{algorithmic}[l]
\STATE Set $k=|\cS|$
\STATE Set $scn=|\cS|$
\WHILE {$k>0$} 
\STATE $N=k$
\STATE clusters = KMedoids($\bm{\mathsf{D}}, k$) \quad \# split $\cS$ into $k$ clusters according to $\bm{\mathsf{D}}$
\FOR{every cluster}
\FOR{every point}
\IF{$\bm{\mathsf{D}}(\text{point, center})>\epsilon$}
\STATE $N = N+1$
\ENDIF
\ENDFOR
\ENDFOR
\STATE $scn = \min\{N, scn\}$
\STATE $k = k-1$
\ENDWHILE
\RETURN{$scn$}
\end{algorithmic}
\end{algorithm}

\section{Experimental Details and Extended Experiments} \label{app:exp}
\subsection{Datasets}
We perform our empirical analysis on CIFAR-10, ShapeNet in Section \ref{sec:experiment} and on CIFAR-100 as well as Restricted ImageNet in Appendix \ref{app:exp_CIFAR-100}.

CIFAR-10 dataset \cite{krizhevsky2009learning} consists of 60000 32x32 color images in 10 classes, with 6000 images per class. There are 50000 training images and 10000 test images. The categories in CIFAR-10 are: \textit{airplane, automobile, bird, cat, deer, dog, frog, horse, ship, truck}.

ShapeNet\footnote{https://shapenet.org/} \cite{chang2015shapenet} is a large-scale 3D model repository. In our experiments, we use a subset of it that contains 10 classes and we resize every image to 32x32. 
There are 30834 training images and 7709 test images.
The categories in this dataset are \textit{sofa, cabinet, chair, display, loudspeaker, lamp, airplane, table, car, watercraft}. 
3D-view transformations could be done by 3D object reconstruction methods, e.g., R2N2 \cite{r2n2}, or rendering tools, e.g., PyTorch3D\footnote{https://PyTorch3D.org/}. 
We use pre-rendered images provided by R2N2\footnote{http://3d-r2n2.stanford.edu/} to approximate the random perturbations of 3D-view.

CIFAR-100 \cite{krizhevsky2009learning} consists of 60000 32x32 color images in 100 classes, with 600 images per class. There are 500 training images and 100 testing images per class.

Restricted ImageNet \cite{tsipras2018robustness} is a subset of ImageNet. It has 8 classes, and each of which is made by grouping a subset of existing, semantically similar ImageNet classes into a super-class. All images are preprocessed into a 64x64 resolution.

\subsection{Data Transformations}
In this paper, we consider \textit{flipping, cropping, rotation} and \textit{3D-view} as data transformations in Section \ref{sec:experiment}. We apply them respectively on one image from the ShapeNet dataset and illustrate the original and transformed images in Figure \ref{fig:illustration_transformation}. 
For flipping, we only consider horizontal flipping. 
For cropping, there are two hyper-parameters, the padding number, and the cropping size, that determine a random cropping operation. An image is first padded with the last value at the edge, and then randomly cropped to a certain size. 
For rotation, we only consider rotating an image around its center. There is one hyper-parameter, degree, that determines a rotation operation. 
For 3D-view transformations, there are three hyper-parameters, distance, elevation, and azimuth, that together determine a specific 3D view. We can interpret the 3D view as a specific position of the camera which is determined by the distance away from the target point, the elevation angle, and the azimuth angle. As long as the camera's position is determined, we would have the 2D image rendered from that specific viewpoint via R2N2 or PyTorch3D. We also evaluate \textit{cutout} and \textit{color jitter} in Appendix \ref{app:norm_scn}. Cutout \cite{cutout} is a data augmentation method that randomly removes contiguous sections of input images. There are three hyper-parameters that control the size, ratio, and pixel values of the rectangle that mask the images. Color jitter is a type of image data transformation where we randomly change the brightness, contrast, and saturation of an image which can be controlled by three hyper-parameters.

\begin{figure}
    \centering
    \includegraphics[width=0.8\textwidth]{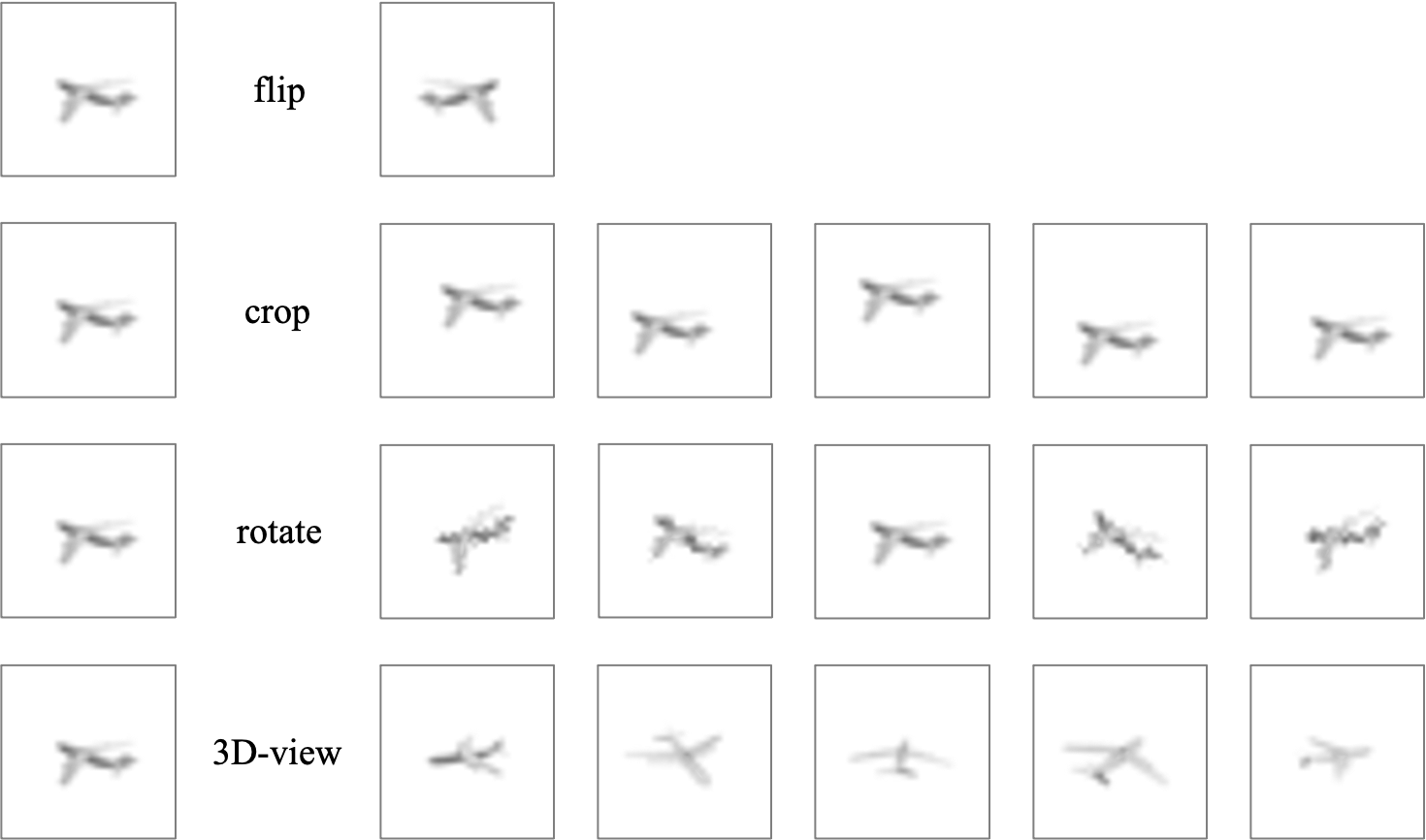}
    \caption{An illustration of data transformations}
    \label{fig:illustration_transformation}
\end{figure}

\subsection{Details on Estimating Sample Covering Numbers}
In this paper, we estimate the sample covering numbers induced by different transformations on CIFAR-10, ShapeNet, CIFAR-100, and Restricted ImageNet. 
Table~\ref{tab:trans_setting} provides the hyper-parameter settings 
that we use for data transformations in this paper. These settings are typically used to preserve labels after data transformations in object classification tasks. 
Continuous data transformations, such as rotation, cropping, and 3D-view, contain infinite numbers of elements in the transformation set. To approximate the orbit, we do sampling every 1 degree for rotation and random sampling (50 times) for cropping, cutout, and color jitter. We use the set of 24 random multi-view images rendered by the R2N2 method to approximate the orbit induced by 3D-view transformations. 

\begin{table}
    \centering
    \begin{tabular}{ll}
    \toprule
         Transformation & Hyper-parameters \\
    \midrule
         Flip & horizontal flip \\
         Rotate & degree $\in [-30,30]$ \\
         Crop & padding $= 4$, cropping size $=32\text{x}32$ \\
         3D-view & distance $\in[0.65,1]$, elevation $ \in [25,30]$, azimuth $\in [0, 360]$ \\
         Cutout & value=0.5, scale=0.05, ratio=1 \\
         ColorJitter & brightness$ \in [0.75,1.25]$, contrast$ \in [0.75,1.25]$, saturation$ \in [0.75,1.25]$ \\
    \bottomrule
    \end{tabular}
    \caption{Data transformations used in our experiments.}
    \label{tab:trans_setting}
\end{table}

\subsection{Details on Evaluating Generalization Benefit}
In Section \ref{exp:benefit}, we evaluate the generalization benefit of learning model invariance to different data transformations. 
We consider the object classification task and use the ResNet18 model architecture on both datasets. 
To learn the invariant models, we use two methods: data augmentation and regularization. 
In the test phase, we evaluate models on clean test sets without applying any data transformations.

\shortsection{Data augmentation method}  
The training loss for the data augmentation method is $\cL_{aug} (\bx) = \cL(f(g(\bx)))$, where $f(\cdot)$ denotes the model and $g(\bx)$ denotes a randomly sampled example in $\bx$'s orbit induced by transformation $\cG$. We use the cross-entropy loss function for $\cL$. 
In each epoch, we randomly sample transformed images as input and preserve ground truth labels.  
We use SGD optimizer with an initial learning rate of 0.01 and decay the learning rate by 0.1 every 50 epochs. 
We train each model for 110 epochs and select the best model according to test accuracy. We run independent experiments four times and report the results in Table \ref{tab:cifar} and \ref{tab:shapenet}. 

\shortsection{Regularization method}
The training loss for regularization method is $\cL_{reg}=\cL_{cls}+\cL_{inv}=\cL(f(\bx))+\lambda\text{KL}(f(\bx), f(g(\bx)))$.
Specifically, in addition to minimizing the classification loss on the original image, we also regularize the model by minimizing the KL divergence between the model's logit outputs on the original image and on the transformed one.
The loss function and optimization settings are the same as those in the data augmentation method except for the case when $\lambda=100$. We use a learning rate of 0.001 without weight decay and train the model for 500 epochs in that experiment. 
At test time, we use two metrics to evaluate the model invariance under 3D-view transformations. The first one is the invairance loss, namely $\cL_{inv} (\bx) =\EE_{g_1,g_2\in\cG} [\text{KL}(f(g_1(\bx)), f(g_2(\bx)))]$.
We approximate the expectation by averaging the KL divergence over the 24 pre-rendered random multi-view images for each original image.
The second metric is $\cA_{inv}$, namely the consistency accuracy under the worst-case transformation. We have $\cA_{inv}(\bx)=1$ if model's outputs on data points in $\bx$'s orbit are consistent, and $\cA_{inv}(\bx)=0$ otherwise. We also use the 24 pre-rendered multi-view images of $\bx$ to approximate its orbit. 
We run independent experiments four times and report the results in Table \ref{tab:lambda}. 

\subsection{Extended Experiments} 
\subsubsection{Experiments on Additional Datasets}\label{app:exp_CIFAR-100}
To better show the consistency between our theory and practice, we conduct additional experiments on CIFAR-100 \cite{krizhevsky2009learning} and Restricted ImageNet \cite{tsipras2018robustness}. We randomly sample 1000 examples in the training set to evaluate the sample covering numbers induced by different data transformations. The settings of data transformations are the same as that in Table \ref{tab:trans_setting}. We train a ResNet18 with different data augmentations three times and report results in Table \ref{tab:CIFAR-100} and \ref{tab:restricted}. The results on CIFAR-100 and Restricted ImageNet both support that a small sample covering number correlates with a small generalization gap.

\begin{table}
\centering
\begin{tabular}{lccccc}
\toprule
            & \multicolumn{3}{c}{Sample covering number} & \multicolumn{2}{c}{Generalization}\\ \cmidrule{2-6} 
Model       & $\epsilon=5.7$       & $\epsilon=7.5$  &  $\epsilon=9.4$   & acc (\%)       & gap \\ \midrule
Base        & 1000 & 990 & 950 & $60.06\pm 0.39$ & $39.91\pm0.40$ \\
Flip        & 1000 & 984 & 945 & $66.49\pm0.46$  & $33.48\pm0.45$  \\
Rotate      & 1000 & 976 & 921 & $67.79\pm0.46$  & $32.17\pm0.47$  \\
Crop        & 995 & 965 & 863 & $72.44\pm0.16$  & $27.53\pm0.16$  \\
\bottomrule
\end{tabular}
\caption{Sample covering numbers, classification accuracy, and generalization gap (the difference between training and test accuracy) for ResNet18 on CIFAR-100.}
\label{tab:CIFAR-100}
\end{table}
\begin{table}
\centering
\begin{tabular}{lccccc}
\toprule
            & \multicolumn{3}{c}{Sample covering number} & \multicolumn{2}{c}{Generalization}\\ \cmidrule{2-6} 
Model       & $\epsilon=14.6$       & $\epsilon=18.4$  &  $\epsilon=21.6$   & acc (\%)       & gap \\ \midrule
Base        & 1000 & 990 & 955 & $82.85\pm 0.42$ & $17.14\pm0.42$ \\
Flip        & 999 & 986 & 941 & $88.07\pm0.39$  & $11.92\pm0.39$  \\
Rotate      & 998 & 967 & 883 & $88.61\pm0.16$  & $11.14\pm0.28$  \\
Crop        & 995 & 947 & 793 & $91.38\pm0.26$  & $8.37\pm0.26$  \\
\bottomrule
\end{tabular}
\caption{Sample covering numbers, classification accuracy, and generalization gap (the difference between training and test accuracy) for ResNet18 on Restricted ImageNet.}
\label{tab:restricted}
\end{table}

\subsubsection{Normalization of Sample Covering Numbers} \label{app:norm_scn}
As discussed in Section \ref{sec:select-trans}, the proposed sample covering number is a model-agnostic measure that does not consider the potential Lipschitz constant increase induced by data transformations. For example, darkening all the images leads to a small sample covering number since the values of all images decrease. However, the Lipschitz constant required for the model is increased to classify closer classes. To mitigate this limitation, we can do normalization for sample covering numbers. Intuitively, the minimum inter-class distance among all class pairs gives us a clue for the required Lipschitz constant. Therefore, we use the ratio between the minimum inter-class before and after applying data transformations to normalize sample covering numbers. In Table \ref{tab:norm_scn}, we evaluate 5 types of data transformations including cutout and color jitter. The sample covering number of color jitter is quite small because it shrinks all the values of images. After normalizing with the minimum inter-class distance, it is larger than that of cropping which aligns with the actual generalization benefits. This is a heuristic normalization that takes potential Lipschitz constant change into consideration. It has limitations such as the normalized sample covering number could exceed the base one. We leave a better normalization for future work.
\begin{table}
    \centering
    \resizebox{\columnwidth}{!}{
    \begin{tabular}{lcccccccc}
\toprule
            & \multicolumn{3}{c}{SCN} & \multicolumn{3}{c}{Normalized SCN} & \multicolumn{2}{c}{Generalization}\\ \cmidrule{2-9} 
Model       & $\epsilon=4.9$       & $\epsilon=6.2$  &  $\epsilon=7.6$ & $\epsilon=4.9$       & $\epsilon=6.2$  &  $\epsilon=7.6$   & acc (\%)       & gap \\ \midrule
Base        & 1000 & 992 & 954 & 1000 & 992 & 954 & $85.43\pm 0.35$ & $14.57\pm0.35$ \\
ColorJitter        & 927 & 710 & 372 & 1000 & 994 & 963 & $85.82\pm0.33$  & $14.18\pm0.33$  \\
Cutout        & 999 & 974 & 902 & 1000 & 993 & 963 & $87.24\pm0.23$  & $12.75\pm0.23$  \\
Flip        & 999 & 990 & 946 & 1000 & 995 & 964 & $89.67\pm0.24$  & $10.33\pm0.24$  \\
Rotate      & 999 & 976 & 909 & 1000 & 988 & 939 & $89.91\pm0.13$  & $10.05\pm0.16$  \\
Crop        & 996 & 961 & 863 & 999 & 985 & 909 & $92.52\pm0.08$  & $7.48\pm0.08$  \\
\bottomrule
\end{tabular}
}
    \caption{Sample covering number (SCN) without and with normalization and generalization performance of ResNet18 on CIFAR-10.}
    \label{tab:norm_scn}
\end{table}

\subsubsection{Estimating Sample Covering Numbers with Different Sample Sizes}
In Section \ref{exp:scn}, we estimate the sample covering numbers on randomly chosen subsets of the whole training datasets.
The sample sizes are 1000 for CIFAR-10 and 800 for ShapeNet.
To investigate the impact of sample sizes on estimation, we further estimate the sample covering numbers with different sample sizes on ShapeNet. 
The results, shown in Figure~\ref{fig:diff_n} (a)-(c), show consistent trends and comparisons among different data transformations in all sample size settings. 
Notably, the 3D-view transformation outperforms other types of transformations by a large margin (and indeed yields better generalization benefits as shown in Table~\ref{tab:shapenet}). 
Therefore, for guiding the data transformation selection, these results suggest that it suffices to estimate the sample covering number on a small subset of the whole dataset for efficiency.

In addition, Figure~\ref{fig:diff_n} (d) shows that the normalized sample covering number decreases as the sample size $n$ increases for fixed $\epsilon$.
This result also suggests that we can keep a fixed ratio between the sample covering number and the sample size but gradually shrink the resolution $\epsilon$ as the sample size $n$ grows.
For a sufficiently large sample size, it is possible to use a very small resolution $\epsilon$ to get a sample covering number that is much smaller than the sample size.
\begin{figure}
    \centering
    \subfigure[$n=100$]{\includegraphics[width=0.4\textwidth]{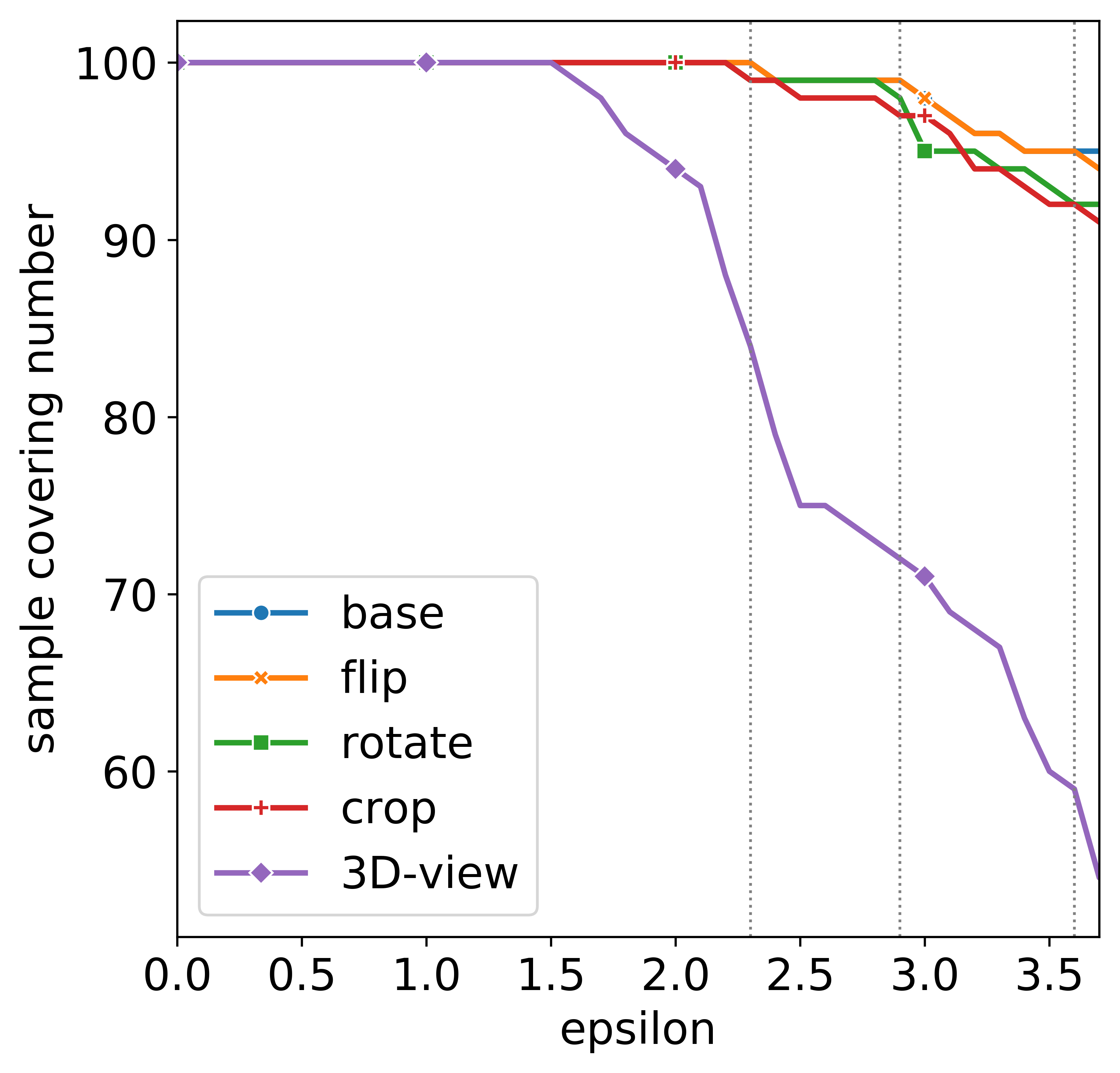}}
    \hfill
    \subfigure[$n=800$]{\includegraphics[width=0.4\textwidth]{fig/scn_r2n2.png}}
    \\
    \subfigure[$n=10000$]{\includegraphics[width=0.4\textwidth]{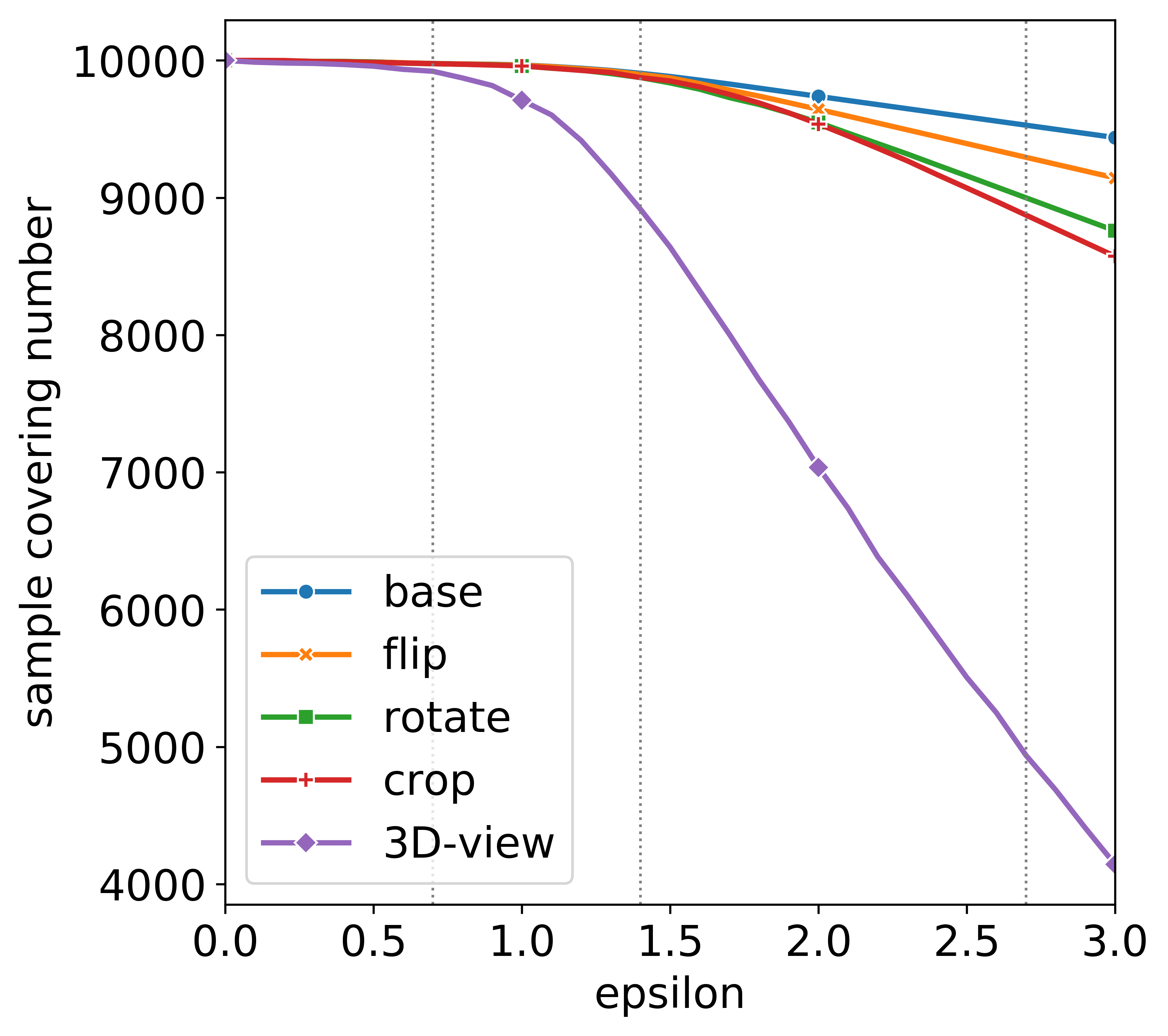}}
    \hfill
    \subfigure[3D-view]{\includegraphics[width=0.4\textwidth]{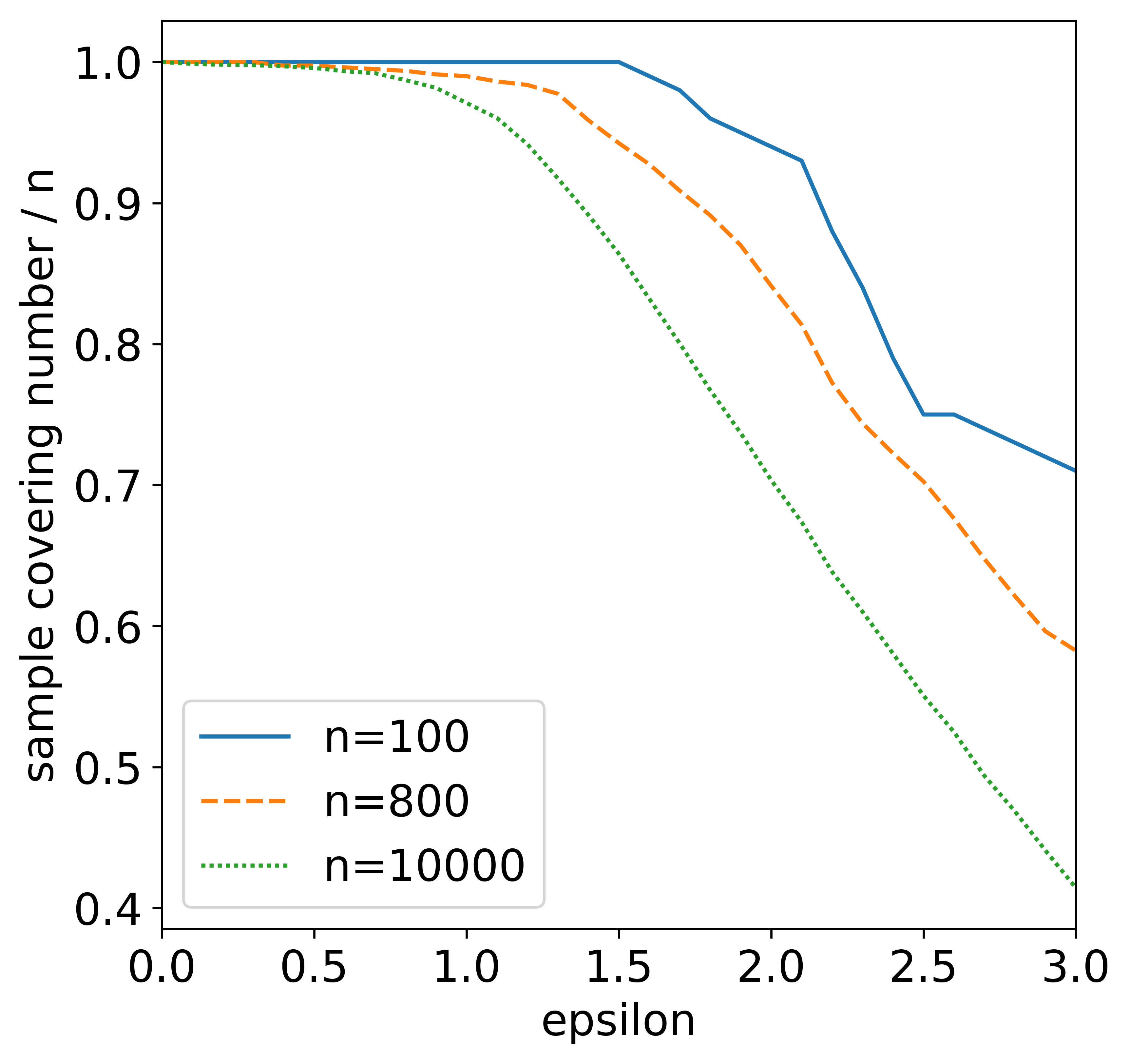}}
    \caption{(a)-(c): Estimated sample covering numbers induced by different data transformations on ShapeNet. $n$ denotes the total sample size. (d): The normalized sample covering number (=sample covering number $/ n$) of 3D-view transformations estimated with different sample sizes. }
    \label{fig:diff_n}
\end{figure}

\subsubsection{Influence of Model Class's Implicit Bias on Generalization Benefit}
\begin{table}
\resizebox{\columnwidth}{!}{
\begin{tabular}{lcccccc}
\toprule
            & \multicolumn{2}{c}{$n=100$} & \multicolumn{2}{c}{$n=1000$} & \multicolumn{2}{c}{$n=all$} \\ \cmidrule{2-7} 
Model       & acc (\%)       & gap        & acc (\%)       & gap         & acc (\%)       & gap        \\ \midrule
Base        & $64.25 \pm 1.87 $ & $20.88\pm2.00$ & $77.50\pm0.48$ & $21.70\pm 0.49$ & $86.67\pm0.37$ & $12.23\pm0.37$ \\
Flip        & $65.00 \pm 2.00$  & $13.84\pm1.90$ & $78.15\pm0.50$ & $16.26\pm0.50$  & $87.22\pm0.32$ & $9.21\pm0.32$ \\
Rotate      & $63.50 \pm 2.14$  & $4.88\pm2.15$ & $76.70\pm0.58$ & $8.98\pm0.55$  & $87.00\pm0.34$ & $5.12\pm0.36$ \\
Crop        & $54.56 \pm 1.96$  & $-4.00\pm1.80$ & $69.60\pm0.42$ & $2.20\pm0.42$  & $83.55\pm0.32$ & $1.58\pm0.36$ \\
3D-View     & $64.75 \pm 1.88$  & $2.25\pm1.88$ & $79.20\pm0.45$ & $3.18\pm0.43$  & $88.28\pm0.28$ & $2.00\pm0.30$ \\ 
\bottomrule
\end{tabular}
}
\caption{Classification accuracy and generalization gap (the difference between training and test accuracy) for MLP on ShapeNet. $n$ denotes the sample size per class. }
\label{tab:mlp}
\end{table}

Our proposed sample covering number is a model-agnostic metric to measure the potential generalization benefit of being invariant to certain data transformations. 
Thus, a natural question is: do all models enjoy the same benefit? 
Different from the ResNet architecture which contains a lot of human priors and engineering work, the 2-layer MLP is among the simplest neural network architectures that better eliminate the influence of architecture's implicit bias.
We use the 2-layer MLP which contains 2 hidden layers, each of which has 10000 hidden units. 
We use ReLU activation for the two hidden layers and do not use common techniques such as batch normalization or dropout.
We use the data augmentation method to train the invariance for the model. 
The loss function and optimization settings are the same as that used in ResNet18. 
We run independent experiments four times and report the results in Table \ref{tab:mlp}.

The decreased generalization gaps shown in Table \ref{tab:mlp} suggest that MLP also benefits from being invariant to data transformations. 
Moreover, comparisons of the generalization gaps between different transformations are similar to those on ResNet18, indicating the effectiveness and applicability of our proposed metric.
Despite the reduced generalization gap, however, MLPs trained with invariance suffer from decreased test accuracy in some cases, especially for cropping. 
This may be due to the limited model capacity of the 2-layer MLP learned by SGD.
In summary, our proposed sample covering number shows empirical effectiveness in predicting the generalization benefit in a model-agnostic way.
Based on our results, we advocate for data transformations that have small sample covering numbers (e.g., 3D-view transformation) and suggest learning the invariance under those data transformations for better generalization performance.

\end{document}